%% file: acl_latex.tex
\documentclass[11pt]{article}

\usepackage[final]{acl}
\usepackage{amsmath}
\usepackage{mathtools}
\usepackage{amssymb}
\usepackage{amsthm}
\usepackage{supertabular}

\newtheorem{proposition}{Proposition}[section]
\newtheorem{theorem}{Theorem}[section]

\usepackage{times}
\usepackage{latexsym}
\usepackage{booktabs}

\usepackage[T1]{fontenc}

\usepackage[utf8]{inputenc}

\usepackage{microtype}

\usepackage{inconsolata}

\usepackage{graphicx}

\title{ABLE: Representing and Mapping LLMs via Attribution-Based Large-model Embedding}

\author{
  \textbf{Zirui Wang\textsuperscript{1,2}},
  \textbf{Yusen Hou\textsuperscript{1}},
  \textbf{Shaofeng Liang\textsuperscript{1,2}},
  \textbf{Bowen Tian\textsuperscript{1,2}},
  \\
  \textbf{Yanlin Zhang\textsuperscript{1}},
  \textbf{Wenshuo Chen\textsuperscript{1,2}$^*$},
  \textbf{Yutao Yue\textsuperscript{1,2}$^*$}
\\\\
  \textsuperscript{1}The Hong Kong University of Science and Technology (Guangzhou)
  \\
  \textsuperscript{2}Deep Interdisciplinary Intelligence Lab (\textit{$DI^2$Lab})
  \\
  \texttt{chatoncws@gmail.com, yutaoyue@hkust-gz.edu.cn}
\\
  \small{$^*$Corresponding authors}
}

\begin{document}
\maketitle
\begin{abstract}
Large language models (LLMs) have undergone rapid growth, forming a heterogeneous ecosystem whose model relationships and capability profiles lack comprehensive documentation.
Organizing such an ecosystem requires reusable model representations that are comparable across architectures while still reflecting how each model processes shared inputs. 
We propose \textbf{ABLE} (\textbf{A}ttribution-\textbf{B}ased \textbf{L}arge-model \textbf{E}mbedding), a training-free framework that represents each model in an \textit{attribution space}. 
Given a shared probing corpus, ABLE computes feature attributions, aligns token-level attribution scores into a tokenizer-agnostic word-level representation, and projects the resulting attribution patterns into compact model embeddings. 
These embeddings summarize model-level input-dependence patterns and can be reused across downstream tasks. 
We construct ABLE representations for 239 LLMs and evaluate them on model organization and relationship analysis, model routing, and benchmark score prediction. 
Across these tasks, ABLE captures meaningful model-level structure, achieves stronger relationship prediction than parameter- and output-based baselines, supports competitive model routing, and predicts benchmark scores from frozen representations, while theoretical analysis further supports the stability of the representation\footnote{Our code for ABLE is available at this link: https://github.com/ziiroo1126/ABLE}.
\end{abstract}

\begin{figure}[t]
\centering
\includegraphics[width=\columnwidth]{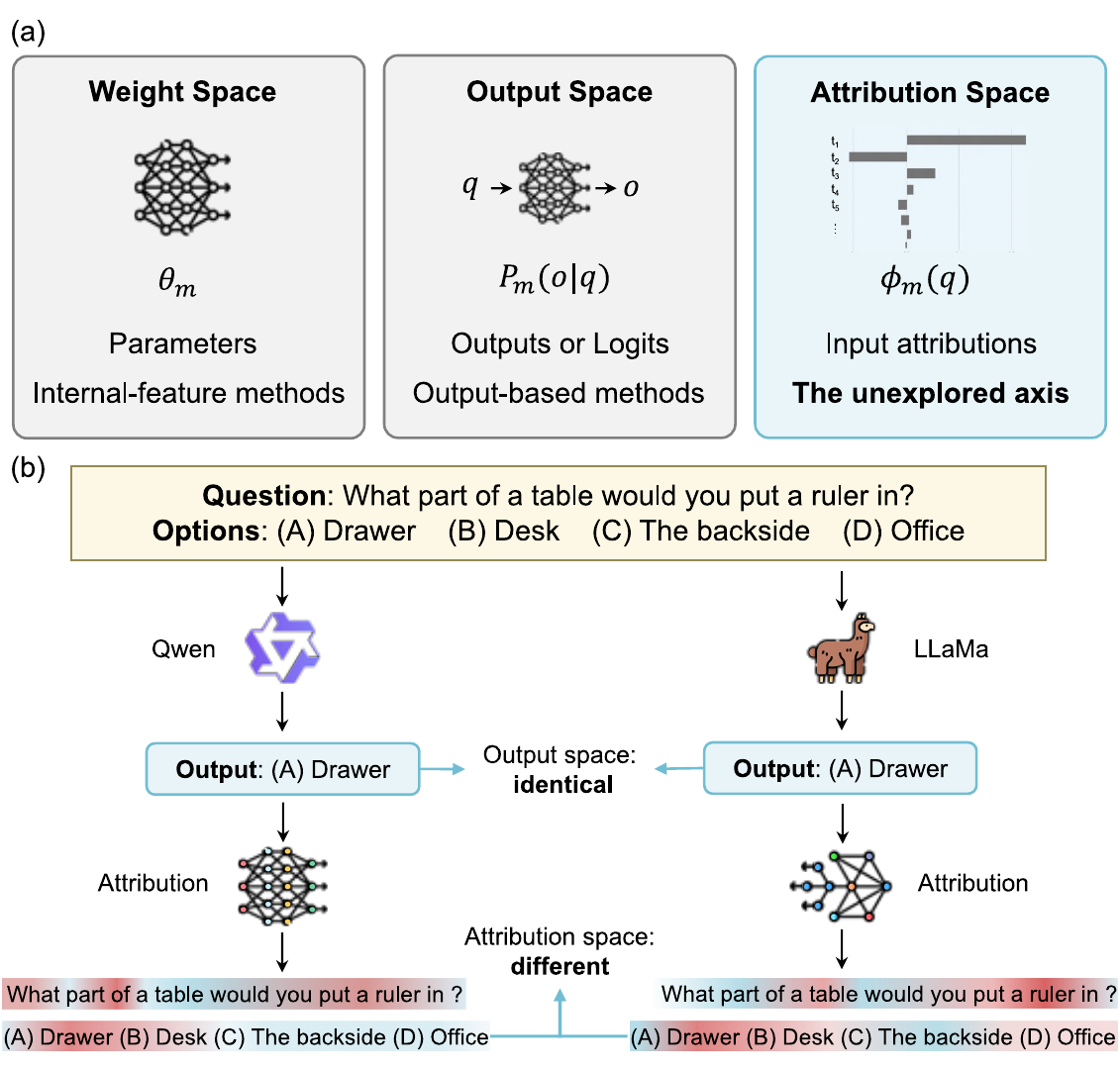}
\caption{(a) \emph{Weight space}, \emph{output space} and our proposed \emph{attribution space} for representing a language model. (b) Two models can produce an identical output yet exhibit clearly
different input attributions.}
\label{fig:overview}
\end{figure}

\section{Introduction}
The ecosystem of large language models (LLMs) is expanding rapidly, with open platforms now hosting hundreds of thousands of models that are continuously fine-tuned, merged, and redistributed \citep{vaswani2017attention, yang2024ecosystem, hao2022language, zhao2023survey, jain2022hugging, rothman2022transformers}. Yet the relationships among these models, such as which model is derived from which and how their capabilities relate, are often only partially documented \citep{barman2024beyond}. Uncovering these relationships would support tasks such as model lineage analysis, model routing, and capability estimation, and in turn benefit broader applications including copyright protection and security auditing \citep{shao2024explanation,ong2024routellm,cheng2024transferring, horwitz2025we}. A natural way to enable these tasks at scale is to embed heterogeneous models into a common representation space where they can be directly compared. A central question, however, is \textit{which aspects of a model such a representation should capture}.

We argue that such a representation should capture \emph{how} a model arrives at its output, namely how it depends on the input evidence available in a shared probing setting. 
To this end, we represent each model in its \textit{attribution space}, defined by feature attributions that trace a model's prediction back to the input features that support it~\citep{wang2024gradient}. 
This view is motivated by a simple observation. 
As shown in Figure~\ref{fig:overview}b, two models can produce the same answer to a question while relying on different parts of the input, so the final answer alone does not fully reveal the model's input-dependence pattern. 
Representing models in attribution space offers \textit{four} properties. 
First, attributions are produced by the model's own computation, grounding the representation in the model rather than in external labels. 
Second, attributions are defined over the shared input, allowing models with different architectures and tokenizers to be placed in one common space. 
Third, they expose how a model depends on its input, an aspect that the answer alone leaves hidden. 
Finally, attributions can be obtained either with or without access to a model's weights, through white-box methods that use gradients or black-box methods that rely only on input and output perturbations, so the attribution space applies in principle to both open-weight and closed-source models. 
As illustrated in Figure~\ref{fig:overview}a, attribution space stands alongside the \textit{weight space} of model parameters or hidden states~\citep{mattheakis2019physical,yadav2023ties, zhu2025independence} and the \textit{output space} of model predictions or logits~\citep{zhuang2024embedllm,yax2024phylolm,oyama2025mapping,wu2026llm}, providing a distinct basis for representing language models.

To build representations in attribution space, we introduce \textbf{ABLE} (\textbf{A}ttribution-\textbf{B}ased \textbf{L}arge-model \textbf{E}mbedding), a training-free framework for constructing model embeddings. 
Given a fixed probing corpus, ABLE computes feature attributions for each model, aligns token-level attribution scores into a shared word-level representation to handle tokenizer heterogeneity, and projects the aligned attribution vector to obtain a compact embedding. 
Each embedding is computed once, summarizes the model's input-dependence pattern over the probing corpus, and can be reused across downstream tasks.

For scalable attribution extraction across hundreds of LLMs, we instantiate ABLE with the Gradient $\times$ Input attribution method~\citep{ancona2019gradient,ancona2017towards,wang2024gradient} and construct representations for 239 LLMs using a probing corpus of 1{,}200 multiple-choice samples. 
We evaluate the resulting embeddings across both structural and practical settings, including model organization and relationship analysis, model routing, and benchmark score prediction. 
Across these tasks, ABLE captures meaningful model-level structure, and we further support the stability of the representation with theoretical analysis.

In summary, our contributions are as follows:
\begin{itemize}
\item We introduce \textit{attribution space} as a distinct basis for representing language models, capturing how a model depends on shared input evidence when producing its outputs.
\item We propose \textbf{ABLE}, a training-free framework that constructs compact model embeddings by computing feature attributions, aligning them across tokenizers into a shared word-level representation, and projecting them into a low-dimensional space.
\item We evaluate ABLE on 239 LLMs across model organization and relationship analysis, model routing, and benchmark score prediction, demonstrating its effectiveness as a reusable model representation and supporting its stability through theoretical analysis.
\end{itemize}

\section{Related Work}

\subsection{LLM Representation}
Prior work has represented and compared LLMs through internal-feature and output-based views.

Internal-feature methods analyze information inside the model, such as parameters or activations. 
Representative approaches study parameter changes from fine-tuning to resolve conflicts during model merging~\citep{yadav2023ties}, test derivative relationships by comparing model weights~\citep{zhu2025independence}, and analyze activation patterns across layers to quantify cross-model similarity~\citep{zhou2024linguistic}. 
These methods provide direct access to internal information and are especially informative when models share compatible architectures.

\begin{figure*}[t]
\centering
\includegraphics[width=0.85\linewidth]{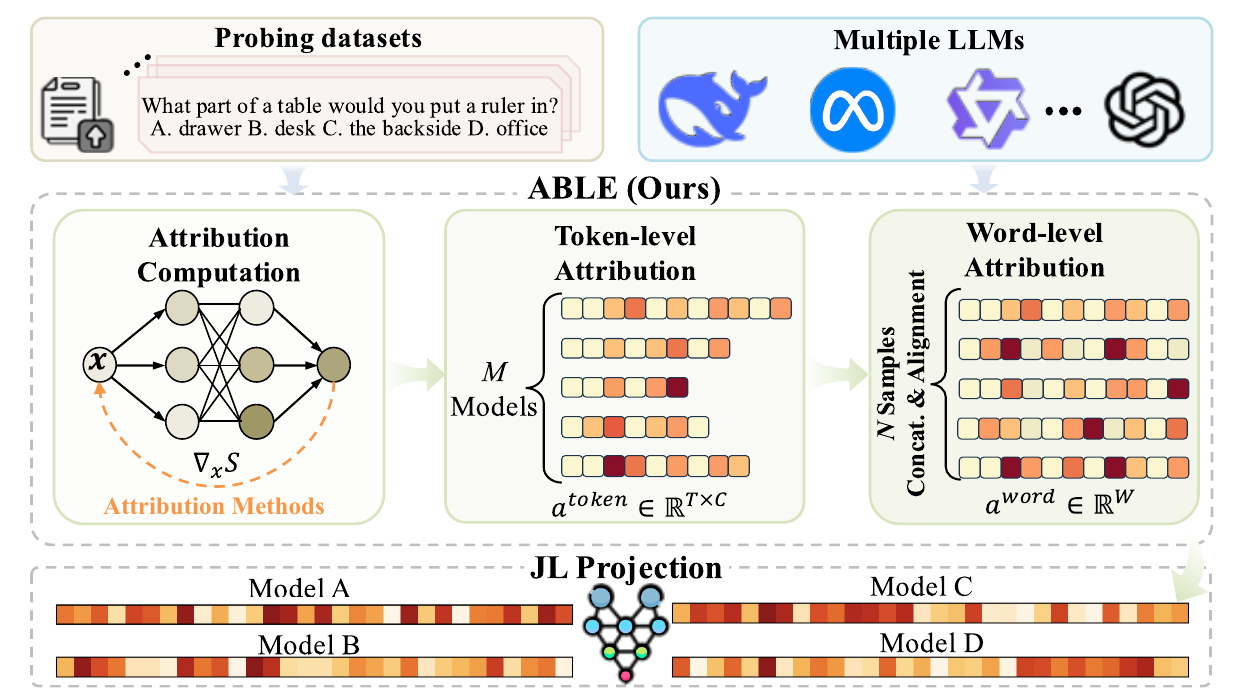}
\caption{Overview of the ABLE pipeline. Given a model, we compute Gradient $\times$ Input attributions on predefined probing datasets, align token-level scores to a unified word vocabulary, and project the result to a compact low-dimensional embedding. JL:  Johnson–Lindenstrauss.}
\label{fig:method}
\end{figure*}

Output-based methods represent models through externally observable behavior, such as logits, generated text, log-likelihoods, or task-performance vectors. 
Prior work constructs model embeddings from log-likelihoods on standardized corpora~\citep{oyama2025mapping}, builds phylogenetic trees by analyzing generated text~\citep{yax2024phylolm,wu2026llm}, and learns model embeddings from task-performance vectors for model routing~\citep{zhuang2024embedllm}. 
These methods are naturally applicable across heterogeneous models and provide scalable signatures for model comparison.

ABLE complements these views by representing models in an interpretable attribution space (Figure~\ref{fig:overview}). 
Rather than representing a model only by parameters or by final outputs, ABLE represents how the model distributes reliance over shared input evidence. 
This attribution-based view yields input-dependence signatures that can be aligned across tokenizers and reused across tasks.

\subsection{Feature Attribution Methods}
Feature attribution methods assign importance scores to input features with respect to a model output, providing interpretable explanations for individual predictions. 
White-box methods leverage internal model information, including Saliency Maps~\citep{simonyan2013deep}, Gradient $\times$ Input~\citep{ancona2019gradient,ancona2017towards}, Integrated Gradients~\citep{sundararajan2017axiomatic}, and Layer-wise Relevance Propagation~\citep{bach2015pixel,samek2021explaining}. 
Black-box methods estimate feature importance through input perturbations or surrogate models, such as LIME~\citep{ribeiro2016should}, which fits local linear models, and SHAP~\citep{lundberg2017unified}, which applies Shapley values.

In the context of LLMs, attribution methods have been widely used for sample-level explanations, such as identifying salient tokens in classification~\citep{liu-avci-2019-incorporating} or generation tasks~\citep{zhao2024reagent}. 
ABLE uses attribution in a different role: instead of treating attribution only as an explanation for a single prediction, it aggregates attribution patterns over a shared probing corpus to construct a reusable representation of the model itself. 
This shifts attribution from sample-level interpretation to model-level representation, enabling comparisons among heterogeneous LLMs in a common attribution space.

\section{Methods}

This section describes the ABLE representation construction pipeline (see Fig.~\ref{fig:method}). Given a collection of $M$ language models and a pre-defined probing dataset, we produce a compact embedding $\mathbf{a}_m \in \mathbb{R}^K$ for each model $m$.

\begin{figure*}[t]
\centering
\includegraphics[width=0.85\linewidth]{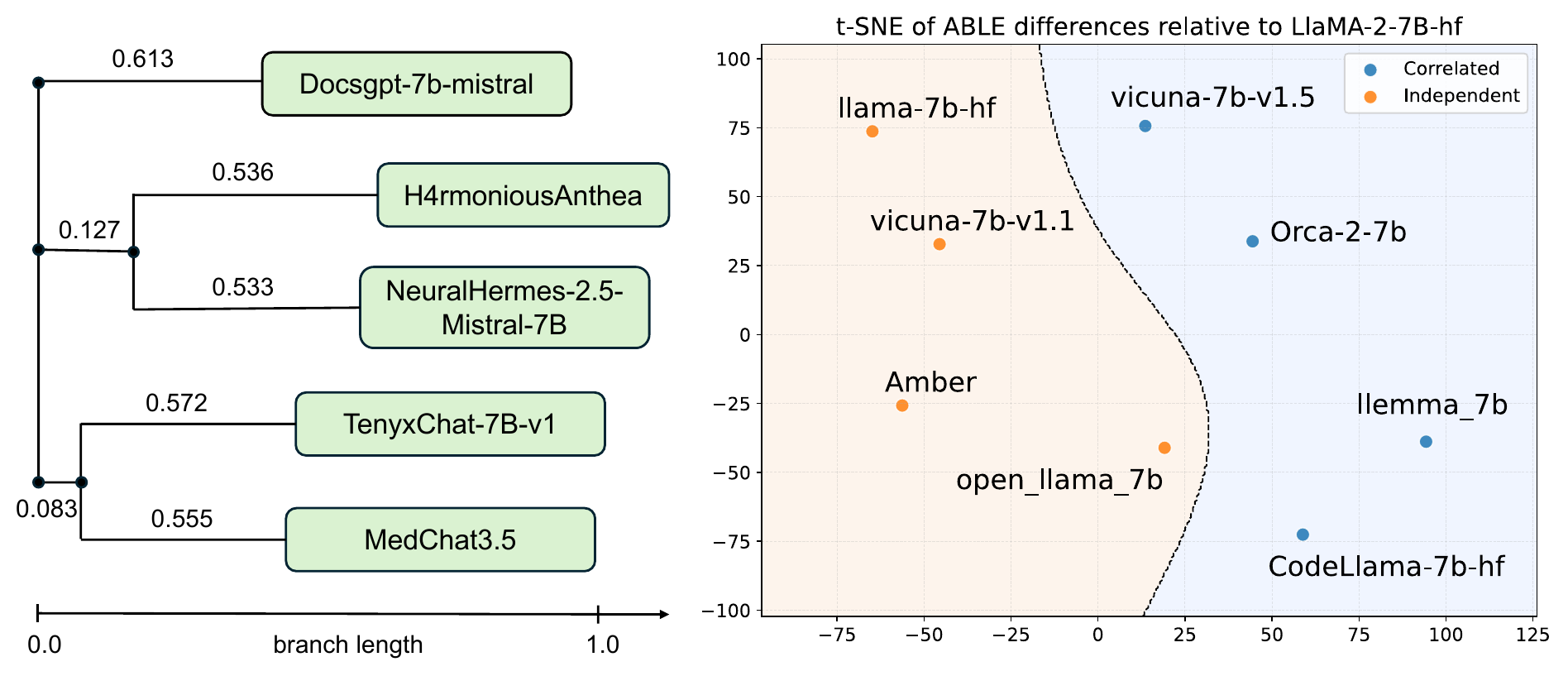}
\caption{Model relationship analysis on documented model families.
Left: Selected Mistral-derived models form an unrooted NJ tree from ABLE embedding distances, consistent with documented fine-tuning relationships. 
Right: Llama-2 family models in ABLE space, showing separation between annotated correlated and independent models.
}
\label{fig:lineage}
\end{figure*}

\subsection{Probing Data Design}
\label{subsec:probing}
We adopt multiple-choice questions (MCQs) from the test sets of standard benchmarks as our probing format. Gold labels are not used during attribution extraction. Each sample consists of a question $q$ and a set of candidate options $\{o_1, o_2, \ldots, o_C\}$, where $C$ is typically 4--5. This design offers two key advantages. First, each option provides a well-defined scalar attribution target: the log-probability of the option sequence given the question, $\log P_m(o_c \mid q)$, can be used for gradient computation. Second, by computing attributions for all options, we capture how the model differentiates among candidate answers, yielding a richer input-dependence pattern than the outputs alone.

\subsection{Attribution Computation}
\label{subsec:attr_comp}
We use Gradient $\times$ Input (GI)~\citep{ancona2019gradient,ancona2017towards,nielsen2022robust} as the default method in all main experiments and study a perturbation-based attribution method in Appendix~\ref{sec:blackbox_attr}. We adopt GI for two reasons: 
(1) computational efficiency, as GI requires only a single forward-backward pass per sample; and 
(2) theoretical tractability, as its gradient-based form allows us to analyze the stability of the resulting attribution representation in Appendix~\ref{sec:theory}.

For each model $m$ and each sample, we compute the attribution of every question token to each option's log-probability. The attribution score for the $t$-th question token with respect to option $o_c$ is:
\begin{align}
    \alpha_t^{(c)} &= \left\langle \mathbf{e}_t, \nabla_{\mathbf{e}_t} S_c \right\rangle, \nonumber \\
    S_c &= \sum_{j=1}^{|o_c|}\log P_m(o_c^{(j)} \mid q, o_c^{(<j)}),
\end{align}
where $\mathbf{e}_t$ denotes embedding of $t$-th question token, and $S_c$ is total log-probability of option $o_c$.

For a question with $T_i^{(m)}$ tokens under model $m$'s tokenizer and $C_i$ options, this yields an attribution matrix 
$\mathbf{A}_{i,m}^{\text{token}} \in \mathbb{R}^{T_i^{(m)} \times C_i}$. 
Flattening this matrix gives an option-aware attribution vector for the sample. 
Concatenating across all $N$ samples produces a model-specific token attribution vector 
$\mathbf{a}_m^{\text{token}} \in \mathbb{R}^{D_{\text{token}}^{(m)}}$, where 
$D_{\text{token}}^{(m)}=\sum_{i=1}^{N}T_i^{(m)}C_i$.

\subsection{Cross-Tokenizer Alignment}
\label{subsec:alignment}
Different models employ different tokenizers, resulting in incompatible dimensions \citep{kudo2018sentencepiece}. To enable cross-model comparison, we align all attributions to a unified word-level vocabulary through a two-step process.

\textbf{Token-to-Character Mapping.} For each token $t$ in the flattened attribution vector $\mathbf{a}_m^{\text{token}}$, let $a_t$ denote its attribution score. We identify its character span $[s_t, e_t)$ in the original text and uniformly distribute the attribution across all characters in this span:
\begin{equation}
    a_{\text{char}}^{(i)} = \frac{a_t}{e_t - s_t}, \quad \forall i \in [s_t, e_t).
\end{equation}
This produces a character-level attribution vector $\mathbf{a}_m^{\text{char}} \in \mathbb{R}^{L}$, where $L$ is the total number of characters in the input text.

\begin{figure*}[t]
\centering
{\includegraphics[width=\linewidth]{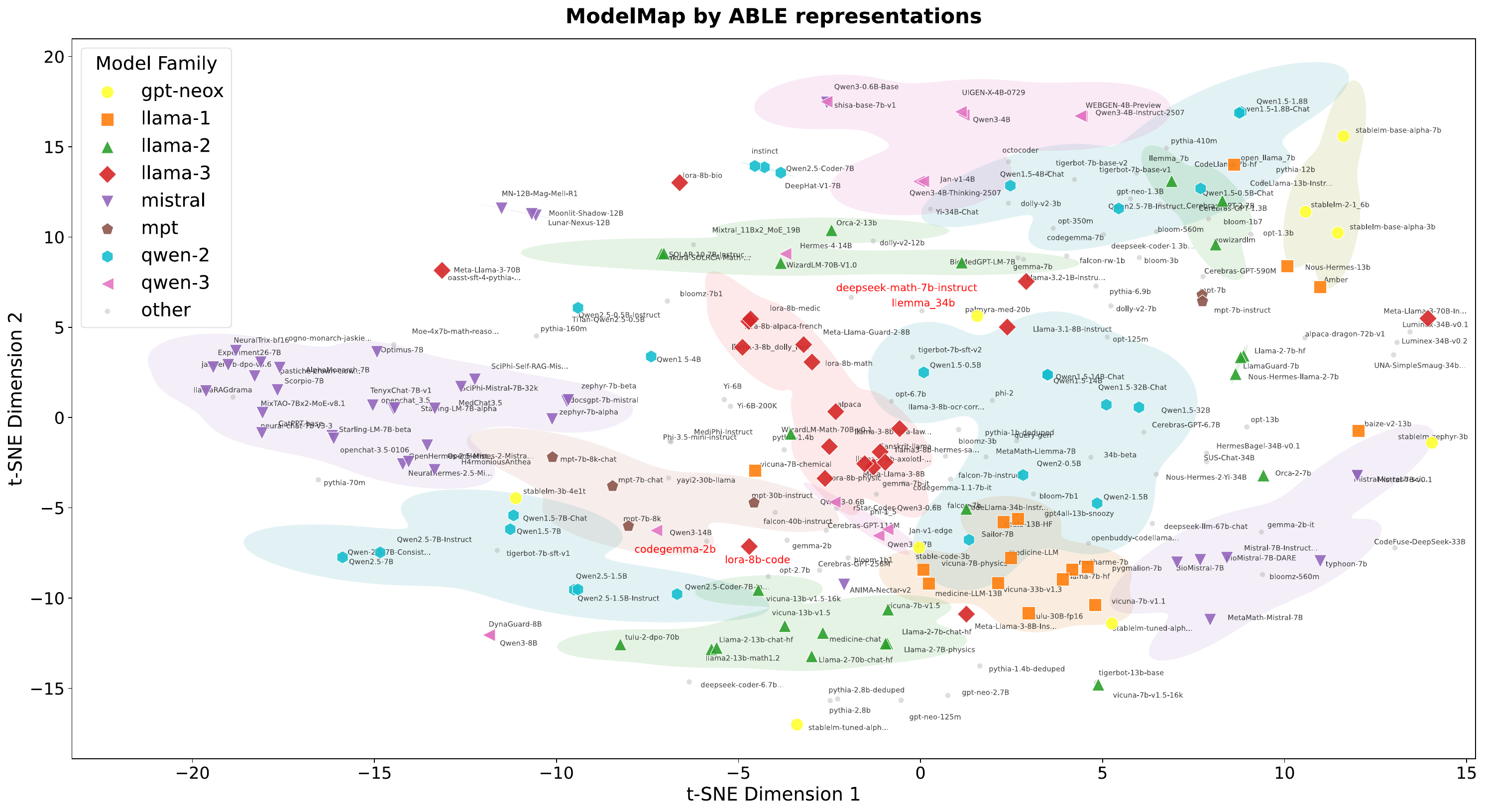}}
\caption{Model Map visualization using t-SNE projection of ABLE embeddings. Models are colored by family, and tightly clustered groups are highlighted with background shading using DBSCAN clustering. Names highlighted in red denote specific examples where models deviate from their families to cluster by functional affinity}
\label{fig:model_map}
\end{figure*}

\textbf{Character-to-Word Aggregation.} Given the character-level attributions, we aggregate them into word-level representations. We segment the character sequence into words using whitespace as the delimiter. For a word $w$ spanning characters $[s_w, e_w)$, its attribution is computed as the sum of its constituent characters:
\begin{equation}
    a_{\text{word}}^{(w)} = \sum_{i=s_w}^{e_w - 1} a_{\text{char}}^{(i)}.
\end{equation}
Whitespace attributions are appended to the preceding word to ensure no attribution is lost.

For each sample-option pair, we apply the token-to-character and character-to-word mapping separately, so that option-specific attribution channels are preserved. 
Since all models are evaluated on the same probing text, this procedure yields a shared word-option attribution coordinate system. 
Concatenating over all options and all samples gives 
$\mathbf{a}_m^{\text{word}} \in \mathbb{R}^{D_{\text{word}}}$, where 
$D_{\text{word}}=\sum_{i=1}^{N} C_i W_i$ and $W_i$ is the number of words in the $i$-th probing input. 
This alignment converts tokenizer-dependent token attributions into a word-level representation.

\subsection{Dimensionality Reduction via Random Projection}
\label{subsec:jl}
The word-level attribution vector $\mathbf{a}_m^{\text{word}}$ can be high-dimensional ($D_{\text{word}}\sim10^5$). To obtain a compact representation, we apply the Johnson-Lindenstrauss (JL) random projection \citep{johnson1984extensions}. Specifically, we sample a random matrix $\mathbf{R} \in \mathbb{R}^{K \times D_{\text{word}}}$ with independent and identically distributed entries drawn from $\mathcal{N}(0, 1/K)$, and compute: $\mathbf{a}_m = \mathbf{R} \, \mathbf{a}_m^{\text{word}} \in \mathbb{R}^K$.
    
By the JL lemma, pairwise distances are preserved up to a multiplicative factor $(1 \pm \epsilon)$ with high probability, provided $K = O(\epsilon^{-2} \log M)$. In practice, we determined the optimal value of $K$ via downstream task performance (see Appendix~\ref{sec:appendixb}). 

\section{Experiments}
We evaluate ABLE through a series of experiments designed to assess both the structural validity and practical utility of its model representations. 
We first examine whether ABLE captures meaningful model-level structure through known-family structure and relationship analysis (Sections~\ref{subsec:lineage} and \ref{subsec:relation}). 
We then evaluate its utility in downstream settings, including model routing and benchmark score prediction (Sections~\ref{subsec:routing} and \ref{subsec:benchmark}).

\subsection{Experimental Setup}
\label{subsec:setup_exp}

\begin{figure*}[t]
\centering
\scalebox{0.90}{\includegraphics[width=\linewidth]{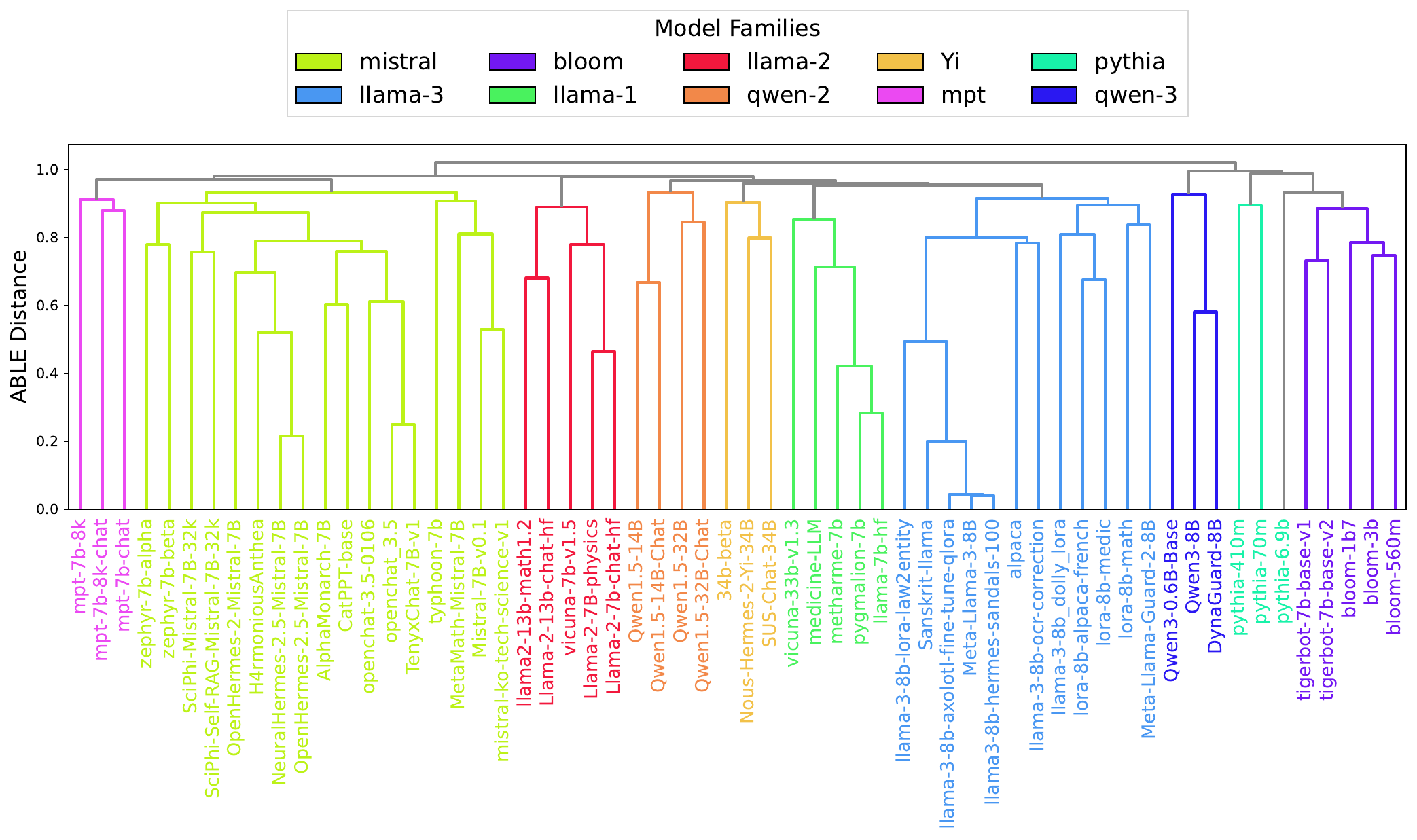}}
\caption{Hierarchical clustering of representative model families using ABLE embeddings. Branches are colored by model family. The tree visualizes similarity structure rather than a literal evolutionary history.}
\label{fig:phylo_tree}
\end{figure*}

\textbf{Datasets.} We construct a balanced probing dataset $\mathcal{D}$ by randomly sampling 200 instances from each of six benchmarks: ARC-Challenge \citep{clark2018think}, Winogrande \citep{sakaguchi2021winogrande}, MMLU \citep{hendryckstest2021, hendrycks2021ethics}, Hellaswag \citep{zellers2019hellaswag}, GPQA \citep{rein2024gpqa}, and CommonsenseQA \citep{talmor-etal-2019-commonsenseqa}, yielding $N = 1{,}200$ samples in total. This combination ensures diverse coverage across reasoning and knowledge domains. For model routing experiments, we additionally use the evaluation dataset from EmbedLLM \citep{zhuang2024embedllm}.

\textbf{Models.} We evaluate $239$ open-weight, decoder-only autoregressive language models whose checkpoints and tokenizers were accessible. We seek coverage across model families, parameter scales, and training types, while deliberately including documented base--derivative cohorts for relationship evaluation; models were not selected based on ABLE's performance. The collection spans 24 families, model sizes from 0.07B to 72B, and releases from 2022 to 2025. Detailed selection criteria are discussed in Appendix~\ref{sec:model_list}.

\textbf{Computational Cost.} 
ABLE requires a one-time attribution extraction step. We report the extraction cost and comparisons with output-based baselines in Appendix~\ref{sec:runtime}.

\subsection{Model Relationship Analysis and Atlas}
\label{subsec:lineage}
To evaluate whether ABLE representations capture meaningful model-level structure and documented relationships among LLMs, we present experiments at two scales: first, small-scale verification on model families with known relationships; second, constructing a comprehensive atlas of the full model collection.

\textbf{Verification on Mistral and Llama-2 Families.} 
To examine whether ABLE reflects documented relationships within model families, we conduct two small-scale case studies. 
For the Mistral family, we select five leaf models with publicly disclosed fine-tuning histories (DocsGPT, NeuralHermes-2.5, TenyxChat-7B, MedChat3.5, and H4rmoniousAnthea) and construct an unrooted tree using the Neighbor-Joining method~\citep{trees1987neighbor} on cosine distances computed from ABLE embeddings. 
As shown in Fig.~\ref{fig:lineage} (left), the resulting tree is consistent with the documented fine-tuning relationships disclosed by model creators~\citep{docsgpt2023, neuralhermes2024, tenyxchat2024, medchat2024, h4rmoniousanthea2024}. 
For the Llama-2 family, we examine eight models, following the relationship annotations by~\citet{zhu2025independence}. 
As shown in Fig.~\ref{fig:lineage} (right), correlated and independent models are separated in ABLE space, with a simple SVM boundary distinguishing the two groups.
We further repeat this Llama-2 case study with \textit{perturbation-based black-box attributions}. The resulting embeddings show a similar separation pattern, suggesting that the attribution-space formulation is not tied to GI alone; details are provided in Appendix~\ref{sec:blackbox_attr}. We further show that, even with incomplete provenance metadata, ABLE retrieves candidate relationships corroborated by parameter‑space evidence (Appendix~\ref{sec:undocumented_retrieval}).


\begin{table*}[t]
\centering
\caption{Model relation prediction performance of ABLE and five baselines (Values are mean $\pm$ standard deviation over 5 random seeds). Bold: the best performance.}
\label{tab:relation}
\begin{tabular}{lccccc}
\toprule
\textbf{Method} & \textbf{Accuracy} & \textbf{Precision} & \textbf{Recall} & \textbf{F1} & \textbf{AUC} \\
\midrule
Random & 0.489$\pm$0.124 & 0.367$\pm$0.093 & 0.733$\pm$0.186 & 0.489$\pm$0.124 & 0.550$\pm$0.139 \\
Greedy & 0.674$\pm$0.096 & 0.527$\pm$0.313 & 0.222$\pm$0.157 & 0.309$\pm$0.203 & 0.561$\pm$0.110 \\
PhyloLM & 0.704$\pm$0.141 & 0.596$\pm$0.182 & 0.844$\pm$0.169 & 0.667$\pm$0.090 & 0.893$\pm$0.059 \\
Log-Likelihood & 0.721$\pm$0.041 & 0.586$\pm$0.064 & 0.861$\pm$0.150 & 0.692$\pm$0.070 & 0.891$\pm$0.053 \\
  UTS (adapted)
  & 0.815$\pm$0.052
  & 0.662$\pm$0.076
  & \textbf{0.956$\pm$0.061}
  & 0.778$\pm$0.041
  & \textbf{0.911$\pm$0.067} \\
\textbf{ABLE (Ours)} & \textbf{0.867$\pm$0.042} & \textbf{0.836$\pm$0.077} & 0.756$\pm$0.093 & \textbf{0.790$\pm$0.065} & 0.906$\pm$0.057 \\
\bottomrule
\end{tabular}
\end{table*}

\begin{table*}[t]
\centering
\caption{Model routing performance of ABLE against three baselines (Values are mean $\pm$ standard deviation over 5 random seeds). ABLE router is competitive with EmbedLLM router.}
\label{tab:routing}
\begin{tabular}{lcccc}
\toprule
 & Random & Single-Best & EmbedLLM & \textbf{ABLE} \\
\midrule
Router Accuracy & 0.413 $\pm$ 0.101 & 0.605 $\pm$ 0.000 & 0.665 $\pm$ 0.003 & \textbf{0.676 $\pm$ 0.001} \\
\bottomrule
\end{tabular}
\end{table*}


\textbf{Global Model Map.} 
We then project the ABLE embeddings into a two-dimensional space using t-SNE~\citep{maaten2008visualizing} to construct a Model Map. 
As shown in Fig.~\ref{fig:model_map}, three patterns emerge. 
First, models tend to group by family, including Llama, Qwen, and Mistral. 
Second, fine-tuned models often appear near their corresponding base models. For example, multiple Llama-3 LoRA-tuned variants cluster around Llama-3-8B model. 
Third, some fine-tuned models deviate from their original families and move closer to models with related task specialization. 
For instance, Llemma-34B and DeepSeek-Math-7B-Instruct appear near each other, consistent with their shared focus on mathematical reasoning, while LoRA-8B-Code and CodeGemma-2B show a similar proximity associated with code-related fine-tuning (Fig.~\ref{fig:model_map}). 
These patterns suggest that ABLE embeddings capture both family-level proximity and functional affinity among heterogeneous LLMs.

\textbf{Cross-Family Hierarchical Clustering.} 
We further select representative models from ten families (Figure \ref{fig:phylo_tree}), compute pairwise cosine distances from their ABLE embeddings, and construct a hierarchical tree with branches colored by family. 

As shown in Fig.~\ref{fig:phylo_tree}, the resulting hierarchy reveals several patterns consistent with known architectural and training information. 
First, models within the same family tend to form cohesive subtrees, suggesting that ABLE embeddings preserve family-level similarity. 
Second, Yi models cluster closely with the LLaMA family, consistent with the documented use of a LLaMA-like architecture. 
Similarly, Bloom and Pythia appear on nearby branches, consistent with their shared GPT-style architectures. 
Third, Qwen2 and Qwen3 occupy different regions of the hierarchy: Qwen2 lies closer to the LLaMA-dominated subtree, whereas Qwen3 branches toward the Pythia-Bloom region. 
This separation is consistent with documentation reporting changes in Qwen3, including reasoning-focused training objectives, multi-stage training pipelines, and reinforcement learning strategies~\citep{bai2023qwen, team2024qwen2, yang2025qwen3}. 
Overall, these observations suggest that ABLE representations reflect both family-level similarity and differences associated with model architecture and training methodology.


\subsection{Relation Prediction}
\label{subsec:relation}

This section evaluates whether ABLE embeddings can support pairwise model relationship prediction. Given two models, the task is to predict whether they have a documented model relationship, which is useful for model ecosystem auditing.

To construct the relation dataset, we identify positive model pairs with documented relationships from the 239 models using metadata from Hugging Face, followed by manual verification. We then sample negative pairs at a 2:1 ratio relative to positive pairs, yielding 135 pairs in total. A negative label indicates that no documented relationship was found in the inspected metadata and does not imply that the two models are necessarily independent. The dataset is split 8:2 into train and test sets. For a model pair $(a,b)$, we construct the pair feature $\mathbf{x}_{ab}=[|\mathbf{a}_a-\mathbf{a}_b|;\mathbf{a}_a-\mathbf{a}_b]$  and train an RBF-kernel SVM classifier.


\begin{table*}[t]
\centering
\caption{Benchmark score prediction performance.  We report the Pearson ($r$) and Spearman ($\rho$) correlation coefficients between the reported benchmark scores and the scores predicted by ridge regression using ABLE. }
\label{tab:benchmark}
\begin{tabular}{lccccc}
\toprule
 & ARC & HellaSwag & MMLU & TruthfulQA & WinoGrande \\
\midrule
Pearson $r$ & 0.8781 & 0.7764 & 0.8554 & 0.9127 & 0.8374 \\
Spearman $\rho$ & 0.8898 & 0.8157 & 0.8637 & 0.8222 & 0.8602 \\
\bottomrule
\end{tabular}
\end{table*}


We compare against five baselines: 
(1) Random, which uniformly predicts correlated or independent; 
(2) Greedy, which predicts that all within-organization pairs are correlated and cross-organization pairs are independent; 
(3) Log-Likelihood~\citep{oyama2025mapping}, which computes log-likelihood vectors on a fixed dataset as model features; 
(4) PhyloLM~\citep{yax2024phylolm}, an output-based method that uses similarity distances to other LLMs as a signature vector; and
(5) an adaptation of Unified Topological Signatures
(UTS)~\citep{rottach2025topology}, which summarizes the geometry of model-generated document embeddings; implementation details are provided in Appendix~\ref{sec:uts_adaptation}.
Because parameter-based methods such as MoTHer~\cite{horwitz2025unsupervised} require compatible layer-wise weight shapes, we provide a detailed comparison on same-architecture models in Appendix \ref{sec:mother_comparison}.


As shown in Table~\ref{tab:relation}, ABLE achieves the best performance on Accuracy, Precision, and F1, whereas UTS obtains the highest mean recall and AUC.
Notably, ABLE and other methods exhibit different precision-recall trade-offs: UTS attains the highest recall but with lower precision, whereas ABLE achieves substantially higher precision at the cost of moderate recall. 
This indicates that ABLE favors precision over recall in this setting.
In other words, ABLE makes fewer positive predictions, but those predictions are more likely to correspond to documented model relationships.

\subsection{Model Routing}
\label{subsec:routing}


To evaluate the utility of ABLE embeddings in model routing, we follow the matrix factorization setup of EmbedLLM~\citep{zhuang2024embedllm}, using the same dataset and evaluation metric. In this setup, a learnable projection maps frozen question embeddings into the model feature space, and the probability that a model answers a question correctly is predicted through element-wise interaction. Unlike EmbedLLM, which learns model embeddings end-to-end for the routing task, we keep ABLE embeddings frozen and only learn the alignment from question embeddings to the ABLE model space.


As shown in Table~\ref{tab:routing}, the router based on frozen ABLE embeddings achieves 67.6\% accuracy, slightly higher than the end-to-end trained EmbedLLM router. While the accuracy gain is modest, the main advantage of ABLE lies in its reusable model representations. When new models are added to the pool, their ABLE embeddings can be computed independently, without recomputing embeddings for existing models. In contrast, EmbedLLM learns model embeddings as task-specific parameters, which may require retraining when the model pool changes.

\subsection{Benchmark Score Prediction}
\label{subsec:benchmark}


In this section, we investigate whether ABLE representations contain signals predictive of LLM performance on benchmarks. 
Using scores from five benchmarks on the Hugging Face Open LLM Leaderboard~\citep{open-llm-leaderboard-v2, myrzakhan2024open, open-llm-leaderboard-v1} as labels and ABLE embeddings as input, we evaluate a ridge regression model~\citep{mcdonald2009ridge} with leave-one-out cross-validation.

As shown in Table~\ref{tab:benchmark} and Fig.~\ref{fig:benchmark}, predictions from ABLE embeddings correlate strongly with reported benchmark scores, achieving Spearman $\rho$ values between 0.81 and 0.89. To assess whether probe--target overlap drives these results, we
reconstruct ABLE after excluding each overlapping target benchmark from the probing corpus. The resulting correlations differ from the full-corpus results by at most 0.024 (Appendix~\ref{sec:robustness}). This indicates that rankings induced by ABLE-based predictions are largely consistent with rankings from standard evaluations. These findings suggest that ABLE representations contain linearly recoverable signals related to model capability. 
\begin{figure}[!b]
\centering
\includegraphics[width=0.9\columnwidth]{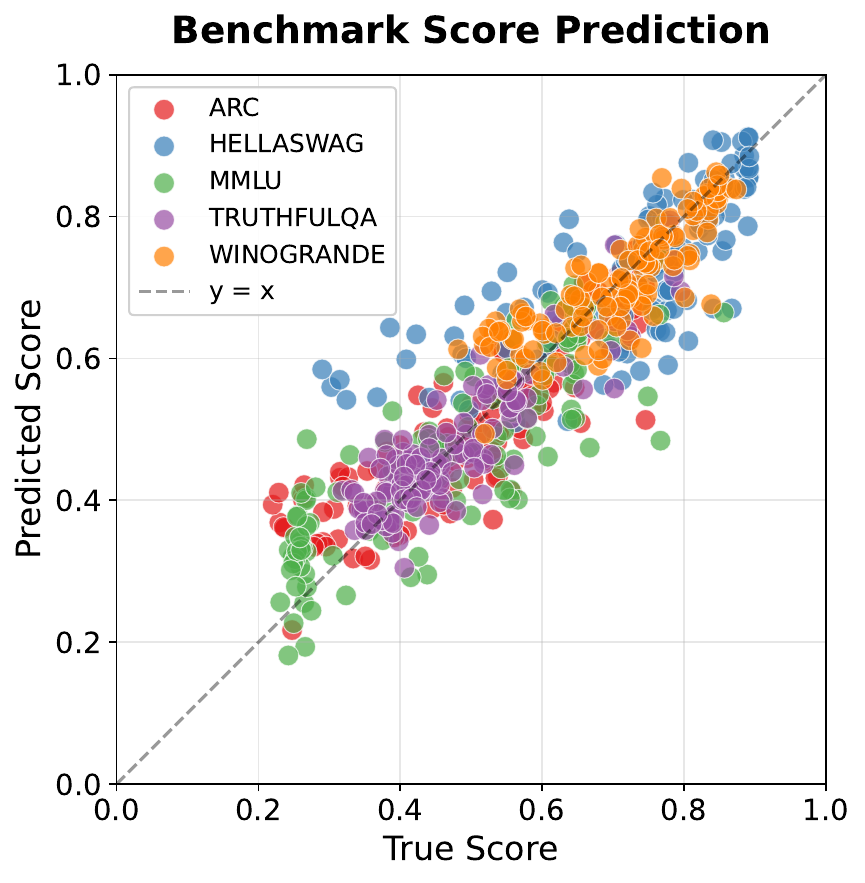}
\caption{Predicted benchmark scores using ABLE against reported benchmark scores on five benchmarks. 
}
\label{fig:benchmark}
\end{figure}
For large-scale model comparison, especially when relative ranking is more important than exact score estimation, ABLE can help prioritize models for further evaluation. We also evaluate ABLE for model capability-ranking analyses detailed in Appendix~\ref{sec:capability_ranking}.


\section{Conclusion}


In this work, we introduced ABLE, a training-free framework for representing language models in attribution space. 
Instead of representing models only through their parameters or outputs, ABLE summarizes how each model depends on shared input evidence over a common probing corpus. 
By computing feature attributions, aligning them across tokenizers into a shared word-level representation, and projecting the resulting attribution patterns into compact embeddings, ABLE provides reusable representations for heterogeneous LLMs.

Experiments on 239 LLMs show that ABLE captures meaningful model-level structure and supports multiple downstream analyses. 
We further provide theoretical analysis supporting the stability of the representation. 
These results suggest that attribution-space representations offer a useful perspective for organizing, comparing, and understanding the rapidly evolving LLM ecosystem.


\section*{Limitations}

\paragraph{Computational Efficiency.} 
Computing GI-based ABLE embeddings requires attribution extraction, which involves one forward and one backward pass per probing sample. 
This makes GI-based ABLE more expensive to extract than purely output-based features. 
However, the extraction is performed once per model, and the resulting embeddings can be reused across downstream tasks. 
We report the extraction cost and runtime comparison in Appendix~\ref{sec:runtime}.

\paragraph{Attribution Method.} 
For scalable extraction across hundreds of LLMs, this work instantiates ABLE with Gradient $\times$ Input. 
We do not systematically evaluate alternative attribution estimators such as Integrated Gradients~\citep{sundararajan2017axiomatic} or SmoothGrad~\citep{smilkov2017smoothgrad}, which may provide smoother or higher-fidelity attribution estimates at the cost of additional computation. 
Exploring the trade-off between attribution fidelity, robustness, and scalability is an important direction for future work.

\paragraph{Closed-Source Models.} 
Our main experiments instantiate ABLE with gradient-based attributions and therefore focus on open-weight models. 
We include a preliminary perturbation-based attribution analysis in Appendix~\ref{sec:blackbox_attr}, showing that a similar separation pattern can be obtained without gradient access in a controlled Llama-2 case study. 
However, extending ABLE to general API-only models still requires robust perturbation estimators and sufficiently informative output signals, which we leave for future work.

\paragraph{Interpreting Model Relationships.} 
ABLE embeddings capture similarity in attribution space, which should not be interpreted as direct proof of model derivation. 
Similar training data, objectives, architectures, or fine-tuning tasks may also lead to similar input-dependence patterns. 
Therefore, relation prediction results should be viewed as evidence for model relatedness or as a screening signal for further analysis, rather than as a definitive provenance or infringement judgment.

\paragraph{Sample Sensitivity.} 
ABLE is computed on a finite probing corpus, so the resulting representation may depend on the choice of probing samples. 
However, as shown in Appendix~\ref{sec:robustness}, ABLE embeddings exhibit high stability across different random subsamples of the probing dataset, suggesting that the representation captures robust model-level patterns rather than being dominated by individual inputs.

\paragraph{Alignment Granularity.} 
Our alignment procedure uniformly distributes token attributions across character spans before aggregating them into word-level scores. 
This simple strategy is effective for the English multiple-choice benchmarks evaluated in this work, but it may lose subword-level information. 
Alternative schemes, such as byte-level alignment, morphology-aware segmentation, or task-specific alignment strategies, may be more suitable for code, multilingual inputs, or domains with non-whitespace tokenization. 
We leave a systematic study of alignment strategies to future work.

\section*{Acknowledgement}
This work was supported by the Guangdong Basic and Applied Basic Research Foundation (Grant No. 2026A1515011579), the HKUST-HKUST(GZ) 1+1+1 Joint Funding Program (Grant No. C\_2025\_031), and the Guangzhou-HKUST(GZ) Joint Funding Program (Grant No. 2023A03J0008), Education Bureau of Guangzhou Municipality. This work was also supported by the Jiangsu Industrial Technology Research Institute (JITRI) and the Wuxi National High-Tech District (WND).

\bibliography{custom}

\appendix

\section{Theoretical Analysis}
\label{sec:theory}

In this section, we provide a theoretical justification that ABLE constitutes a stable attribution-based embedding of large language models. Our analysis abstracts away implementation details and focuses on the relationship between model parameters and attribution-based representations. Section A.1 defines ABLE. Section A.2 establishes that ABLE is a stable feature map induced by model parameters. Section A.3 proves that the low-dimensional embedding preserves inter-model distances via random projection. Section A.4 provides finite-sample concentration guarantees.

\subsection{Preliminaries and Definitions}

Let a language model be parameterized by $\theta \in \mathbb{R}^P$.  
For an input representation $x \in \mathbb{R}^d$ and an output class $y$, denote the log-probability function by
\[
s_\theta^{(y)}(x) \triangleq \log p_\theta(y \mid x),
\]
which we assume to be differentiable with respect to $x$ almost everywhere.

\paragraph{Attribution-based sensitivity.}
For each input $x$, we define the attribution score as the inner product between the input and its gradient:
\[
a_\theta(x) \triangleq \langle x, \nabla_x s_\theta(x) \rangle \in \mathbb{R},
\]
where $s_\theta(x)$ denotes the log-probability function stacking $s_\theta^{(y)}(x)$ over output classes.

\paragraph{Model representation.}
Let $\phi : \mathbb{R}^d \to \mathbb{R}^D$ be a deterministic aggregation function.  
Given a probe distribution $\mathcal{D}$ over inputs, we define the (pre-projection) ABLE representation as
\[
\Phi(\theta) \triangleq \mathbb{E}_{x \sim \mathcal{D}} \big[ \phi(a_\theta(x)) \big] \in \mathbb{R}^D.
\]
Finally, the ABLE embedding is obtained via a random projection
\[
\Psi(\theta) \triangleq \mathbf{R} \, \Phi(\theta) \in \mathbb{R}^K,
\]
where $\mathbf{R} \in \mathbb{R}^{K \times D}$ is a Johnson--Lindenstrauss random matrix.

\subsection{The Parameter Stability of ABLE}

The objective of this section is to demonstrate the parameter stability of ABLE, meaning that the embedding varies continuously with changes in model parameters. Our main line of reasoning is to first prove that ABLE is a Lipschitz continuous mapping under standard regularity assumptions, and subsequently demonstrate that Transformer-based LLMs satisfy these assumptions.

\begin{theorem}[Parameter Stability of ABLE]
\label{thm:lipschitz}
Assume the following regularity conditions hold:
\begin{enumerate}
    \item The gradient field of the log-probability function, $\nabla_x s_\theta(x)$, is $M$-Lipschitz continuous in $\theta$: $\|\nabla_x s_\theta(x) - \nabla_x s_{\theta'}(x)\|_2 \le M \|\theta - \theta'\|_2$.
    \item For all $x$ in the support of $\mathcal{D}$, $s_\theta(x)$ is differentiable with respect to $x$ almost everywhere.
    \item The aggregation function $\phi$ is $L_\phi$-Lipschitz, and inputs are bounded on average: $\mathbb{E}_{x \sim \mathcal{D}} \|x\|_2 \le B$.
\end{enumerate}
Then the ABLE mapping $\Phi : \theta \mapsto \Phi(\theta)$ is Lipschitz continuous:
\[
\| \Phi(\theta) - \Phi(\theta') \|_2 \le L \, \| \theta - \theta' \|_2,
\]
where $L = B M L_\phi$.
\end{theorem}

\begin{proof}
Using the Cauchy-Schwarz inequality and the assumption on $\nabla_x s_\theta(x)$:
\begin{align*}
| a_\theta(x) - a_{\theta'}(x) | &= |\langle x, \nabla_x s_\theta(x) - \nabla_x s_{\theta'}(x) \rangle| \\
&\le \|x\|_2 \|\nabla_x s_\theta(x) - \nabla_x s_{\theta'}(x)\|_2 \\
&\le \|x\|_2 M \| \theta - \theta' \|_2.
\end{align*}
Applying the Lipschitz continuity of $\phi$ and taking expectation over $x \sim \mathcal{D}$:
\begin{align*}
&\| \Phi(\theta) - \Phi(\theta') \|_2 \\
&\le \mathbb{E}_{x \sim \mathcal{D}}[ L_\phi | a_\theta(x) - a_{\theta'}(x) | ] \\
&\le L_\phi \mathbb{E}_{x \sim \mathcal{D}}[\|x\|_2] \cdot M \cdot \| \theta - \theta' \|_2 \\
&\le B M L_\phi \| \theta - \theta' \|_2.
\end{align*}
\renewcommand{\qedsymbol}{}
\end{proof}

Theorem~\ref{thm:lipschitz} provides the formal guarantee for ABLE's parameter stability. The Lipschitz inequality $\|\Phi(\theta) - \Phi(\theta')\|_2 \le L \|\theta - \theta'\|_2$ ensures that the mapping does not diverge: finite differences in model parameters result in strictly bounded differences in the embedding space. This property is crucial for a reliable representation, as it guarantees that models with similar internal mechanisms (proximal parameters) are mapped to nearby points in the vector space, preventing chaotic behavior where minor weight variations could lead to disparate embeddings.

The validity of Theorem~\ref{thm:lipschitz} relies on three assumptions. Assumptions 2 and 3 are satisfied by design in our setting:
\begin{itemize}
    \item \textbf{Assumption 2 (Differentiability):} While text is discrete, the model operates on continuous word embeddings. Standard embedding lookup layers and subsequent smooth transformations ensure differentiability with respect to the input representation $x$ almost everywhere.
    \item \textbf{Assumption 3 (Boundedness):} The aggregation function $\phi$ is linear (averaging) and thus 1-Lipschitz. Input embeddings belong to a finite vocabulary set $\mathcal{V}$, so $\|x\|_2$ is strictly bounded by $B = \max_{v \in \mathcal{V}} \|v\|_2$.
\end{itemize}
The remaining Assumption 1, which requires the model's gradient field to be Lipschitz continuous with respect to its parameters, is the most non-trivial condition. We now demonstrate that this assumption holds for Transformer-based Large Language Models under standard architectural constraints.

\begin{proposition}[Regularity of Transformer Gradient Fields]
\label{prop:transformer_regularity}
Consider a Transformer-based language model with the following properties:
\begin{enumerate}
    \item Activation functions (e.g., GELU, Swish) have bounded first and second derivatives: $|\sigma'(z)| \le L_\sigma$ and $|\sigma''(z)| \le L_{\sigma'}$ for all $z$.
    \item Token embeddings are drawn from a finite vocabulary with bounded norms: $\|x\|_2 \le B_x$.
    \item The network has finite depth $L$ and bounded weight matrices: $\|W_l\|_2 \le B_W$ for all layers $l$.
\end{enumerate}
Then the gradient field $\nabla_x s_\theta(x)$ is $M$-Lipschitz continuous in $\theta$ for some constant $M > 0$.
\end{proposition}

\begin{proof}
We proceed by analyzing the Lipschitz dependence layer by layer.

\textit{Step 1: Linear layer.} For $y = Wx$, the input gradient is $\nabla_x y = W^T$. For two parameter configurations $W, W'$:
\begin{align*}
\|\nabla_x y_W - \nabla_x y_{W'}\|_2 &= \|W^T - W'^T\|_2 \\
&= \|W - W'\|_2,
\end{align*}
which is 1-Lipschitz in $W$.

\textit{Step 2: Layer with activation.} For $y = \sigma(Wx)$, the chain rule gives $\nabla_x y = W^T \cdot \mathrm{diag}(\sigma'(Wx))$. When $W$ changes to $W'$:
\begin{align*}
&\|\nabla_x y_W - \nabla_x y_{W'}\|_2 \\
&\le \|W^T - W'^T\|_2 \cdot \|\mathrm{diag}(\sigma'(Wx))\|_2 \\
&\quad + \|W'^T\|_2 \cdot \|\sigma'(Wx) - \sigma'(W'x)\|_2 \\
&\le L_\sigma \|W - W'\|_2 \\
&\quad + B_W L_{\sigma'} \|x\|_2 \|W - W'\|_2 \\
&\le (L_\sigma + B_W L_{\sigma'} B_x) \|W - W'\|_2.
\end{align*}

\textit{Step 3: Multi-layer composition.} For an $L$-layer network $f = f_L \circ \cdots \circ f_1$, the chain rule yields $\nabla_x f = \prod_{l=1}^{L} J_l$, where $J_l$ is the Jacobian of layer $l$. Since each $J_l$ is Lipschitz in $\theta_l$ with bounded spectral norm (due to bounded weights and the smoothness of softmax attention), perturbation analysis shows that the product $\nabla_x f$ remains Lipschitz in $\theta$.

For Transformer architectures specifically, self-attention layers use softmax, whose Jacobian satisfies $\|\partial \mathrm{softmax} / \partial z\|_2 \le 1$. Combined with bounded query, key, and value projections, each attention layer contributes a bounded Lipschitz factor. Aggregating across all $L$ layers yields the global constant $M$.
\renewcommand{\qedsymbol}{}
\end{proof}

Consequently, ABLE constitutes a stable feature map as implied by its Lipschitz continuity, where small perturbations in model parameters result in bounded changes in the representation space. It is worth noting that this guarantee fundamentally relies on differentiability; architectures employing non-differentiable components (e.g., hard attention) may require alternative theoretical analysis.

\subsection{Distance Preservation via Random Projection}

We next show that the final ABLE embedding preserves inter-model geometry up to a small distortion.

\begin{theorem}[Johnson--Lindenstrauss Property]
\label{thm:jl}
Let $\mathcal{F} = \{\theta_1, \dots, \theta_m\}$ be a finite set of models and let $0 < \varepsilon < 1$.  
If the projection dimension satisfies $K = O(\varepsilon^{-2} \log m)$, then with high probability, for all $i,j \in \{1,\dots,m\}$, the projected distance $\|\Psi(\theta_i) - \Psi(\theta_j)\|_2$ lies within $(1 \pm \varepsilon)$ of the original distance $\|\Phi(\theta_i) - \Phi(\theta_j)\|_2$.
\end{theorem}

\begin{proof}
Let $\Delta_{ij} = \Phi(\theta_i) - \Phi(\theta_j)$ be the difference vector between any two models. Only the linear projection $\Psi(\theta) = \mathbf{R}\Phi(\theta)$ comprises random variables, where entries of $\mathbf{R} \in \mathbb{R}^{K \times D}$ are i.i.d. Gaussian $\mathcal{N}(0, 1/K)$.

Consider a single pair with difference vector $v = \Delta_{ij}$. We are interested in the distribution of $\|\mathbf{R}v\|^2$. Let $R_k$ be the $k$-th row of $\mathbf{R}$. Then the $k$-th component of the projected vector is $y_k = \langle R_k, v \rangle$. Since $R_k \sim \mathcal{N}(0, \frac{1}{K}I)$, the linear combination $y_k$ is essentially a univariate Gaussian:
\[
y_k \sim \mathcal{N}\left(0, \frac{\|v\|^2}{K}\right).
\]
Consequently, the squared norm of the projection is:
\[
\|\mathbf{R}v\|^2 = \sum_{k=1}^K y_k^2 = \frac{\|v\|^2}{K} \sum_{k=1}^K Z_k^2,
\]
where $Z_k \sim \mathcal{N}(0,1)$ are standard Gaussian variables. The sum $X = \sum_{k=1}^K Z_k^2$ follows a Chi-squared distribution with $K$ degrees of freedom, denoted as $\chi^2_K$.

We use standard concentration bounds for the Chi-squared distribution (derived from the Moment Generating Function bound). For any $\epsilon \in (0, 1)$:
\begin{align*}
\mathbb{P}(X \ge K(1+\epsilon)) &\le \exp\left(-\frac{K}{4}\epsilon^2\right), \\
\mathbb{P}(X \le K(1-\epsilon)) &\le \exp\left(-\frac{K}{4}\epsilon^2\right).
\end{align*}
Combining these, the probability that the squared length is distorted by more than $\epsilon$ is:
\begin{align*}
\mathbb{P}\Big( \big| \|\mathbf{R}v\|^2 &- \|v\|^2 \big| > \epsilon \|v\|^2 \Big) \\
&= \mathbb{P}\left( \left| \frac{X}{K} - 1 \right| > \epsilon \right) \\
&\le 2\exp\left(-\frac{K \epsilon^2}{4}\right).
\end{align*}
We apply this to the set of all pairwise differences $\mathcal{V} = \{ \Delta_{ij} \mid 1 \le i < j \le m \}$, which has cardinality $|\mathcal{V}| = \binom{m}{2} < m^2$. Using the union bound:
\begin{align*}
P_{\text{fail}} &= \mathbb{P}\Big( \exists i,j : \left| \|\Psi(\theta_i) - \Psi(\theta_j)\|^2 - \|\Delta_{ij}\|^2 \right| \\
&\qquad\qquad > \epsilon \|\Delta_{ij}\|^2 \Big) \\
&\le \sum_{i<j} \mathbb{P}\left( \left| \|\mathbf{R}\Delta_{ij}\|^2 - \|\Delta_{ij}\|^2 \right| > \epsilon\|\Delta_{ij}\|^2 \right) \\
&\le m^2 \cdot 2\exp\left(-\frac{K \epsilon^2}{4}\right).
\end{align*}
To ensure $P_{\text{fail}} \le \delta$, we require:
\[
2m^2 \exp\left(-\frac{K \epsilon^2}{4}\right) \le \delta \implies K \ge \frac{4}{\epsilon^2} \ln \frac{2m^2}{\delta}.
\]
Thus, choosing $K = O(\epsilon^{-2} \log m)$ suffices to preserve all pairwise distances with high probability.
\renewcommand{\qedsymbol}{}
\end{proof}

\begin{figure*}[t]
\centering
\includegraphics[width=\linewidth]{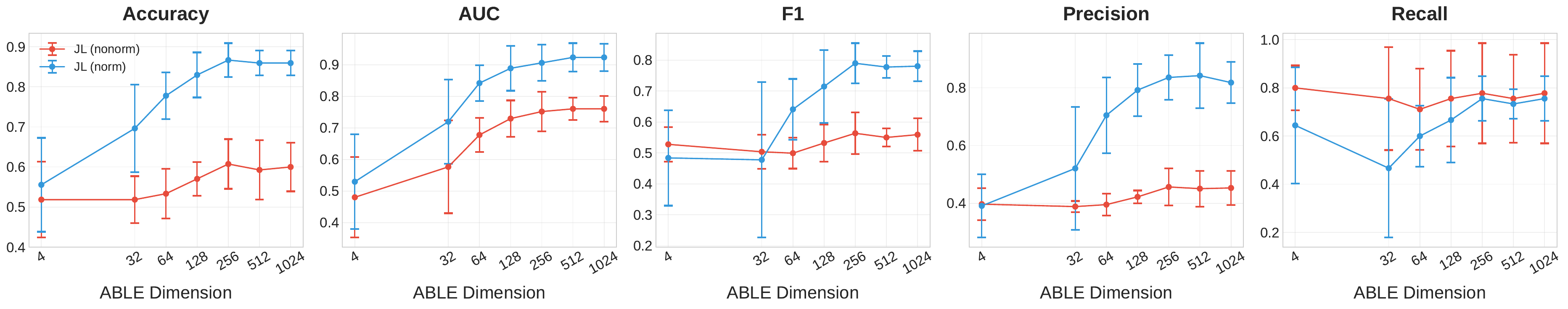}
\caption{Effect of ABLE embedding dimension on relation prediction performance. The blue curve corresponds to applying L2 normalization to each model's feature vector prior to Johnson-Lindenstrauss projection, while the red curve represents the unnormalized variant. Error bars indicate standard deviation computed over 5 independent runs.}
\label{fig:dim_ablation}
\end{figure*}

\paragraph{Consequence.}
Theorem~\ref{thm:jl} implies that ABLE preserves pairwise distances between models up to a controlled distortion, enabling reliable comparison, clustering, and retrieval in a low-dimensional embedding space.

\subsection{Finite-Sample Concentration}

In practice, the ABLE representation is computed from $N$ i.i.d.\ samples rather than the population expectation. We now establish that the empirical estimate concentrates around its expectation, bridging the gap between population-level theory and finite-sample implementation.

\begin{theorem}[Finite-Sample Guarantee]
\label{thm:concentration}
Let $\hat{\Phi}_N(\theta) = \frac{1}{N}\sum_{i=1}^N \phi(a_\theta(x_i))$ be the empirical ABLE representation based on $N$ i.i.d.\ samples from $\mathcal{D}$. 
Assume that $\|\phi(a_\theta(x))\|_\infty \le B_\phi$ for all $x$ in the support of $\mathcal{D}$.
Then for any $\delta \in (0,1)$, with probability at least $1-\delta$:
\[
\|\hat{\Phi}_N(\theta) - \Phi(\theta)\|_2 \le B_\phi \sqrt{\frac{2D \log(2D/\delta)}{N}},
\]
where $D$ is the dimension of the pre-projection representation.
\end{theorem}

\begin{proof}
For each coordinate $j \in \{1,\ldots,D\}$, the empirical mean $[\hat{\Phi}_N(\theta)]_j = \frac{1}{N}\sum_{i=1}^N [\phi(a_\theta(x_i))]_j$ is an average of $N$ i.i.d.\ bounded random variables with $|[\phi(a_\theta(x))]_j| \le B_\phi$. By Hoeffding's inequality:
\[
\mathbb{P}\big(|[\hat{\Phi}_N(\theta)]_j - [\Phi(\theta)]_j| > \epsilon\big) \le 2\exp\Big(-\frac{N\epsilon^2}{2B_\phi^2}\Big).
\]
Setting the right-hand side to $\delta/D$ and applying a union bound over all $D$ coordinates, we obtain that with probability at least $1-\delta$, $|[\hat{\Phi}_N(\theta)]_j - [\Phi(\theta)]_j| \le B_\phi\sqrt{\frac{2\log(2D/\delta)}{N}}$ for all $j$. The $\ell_2$ bound follows from $\|\cdot\|_2 \le \sqrt{D}\|\cdot\|_\infty$.
\renewcommand{\qedsymbol}{}
\end{proof}

\paragraph{Consequence.}
Theorem~\ref{thm:concentration} quantifies the gap between the theoretical population-level representation $\Phi(\theta)$ and its finite-sample estimate $\hat{\Phi}_N(\theta)$. The bound scales as $O(\sqrt{D\log D / N})$, indicating that the estimation error diminishes with more samples. This theoretical guarantee is empirically substantiated in Appendix~\ref{sec:robustness}. There, we calculate pairwise ABLE distances between models using two disjoint subsets of the probing data. The strong correlation observed between these two sets of distances indicates that the finite sample size is sufficient to achieve stable representations, thereby validating the convergence predicted by the concentration bound.

\section{Ablation Study of ABLE Representation Dimensionality}
\label{sec:appendixb}
This ablation study examines how the ABLE embedding dimension $d$ affects representation quality. We vary $d$ across \{4, 32, 64, 128, 256, 512, 1024\} and evaluate each setting on the relation prediction task described in \S\ref{subsec:relation}. We report Accuracy, AUC, F1, Precision, and Recall as evaluation metrics.

As shown in Fig.~\ref{fig:dim_ablation}, representation quality improves as $d$ increases and converges around $d=256$. Very low dimensions (e.g., $d=4$) yield poor performance because excessive compression discards structural information from the original high-dimensional space. Conversely, dimensions beyond 256 offer only marginal gains while incurring higher computational and storage costs. This trade-off between quality and efficiency motivates our choice of $d=256$ for all experiments.

We also compare two variants: with and without L2 normalization of each model's feature vector prior to JL projection. Normalization rescales each feature vector to unit length, i.e., $\mathbf{a}'_m = \mathbf{a}_m / \|\mathbf{a}_m\|_2$. As shown in Fig.~\ref{fig:dim_ablation}, the normalized variant (blue) generally outperforms the unnormalized variant (red) across most metrics. This suggests that normalization balances the scale across dimensions, enabling the distance metric to better capture directional differences between models.

\section{Empirical Validation via Model Interpolation}
\label{sec:appendixc}
The theoretical analysis in Appendix~\ref{sec:theory} establishes that ABLE is a Lipschitz-continuous mapping from model parameters to the embedding space: models with similar parameters should yield similar ABLE embeddings. To empirically validate this correspondence on real nonlinear Transformers, we conduct a model interpolation experiment.

We select three Llama-2-7B variants sharing the same architecture as anchor models: the base model ($W_0$), Llama-2-7b-chat-hf ($W_1$), and Vicuna-7b-v1.5 ($W_2$). We construct 25 merged models via linear weight interpolation:
\[
W_{\alpha,\beta} = (1 - \alpha - \beta) W_0 + \alpha W_1 + \beta W_2,
\]
where $\alpha, \beta \in \{0, 0.25, 0.5, 0.75, 1.0\}$. In weight space, we reconstruct the 2D geometric structure using inner products among the three anchor models; in ABLE space, we apply Principal Component Analysis (PCA) to the low-dimensional ABLE features for 2D visualization.

\begin{figure}[t]
\centering
\includegraphics[width=\columnwidth]{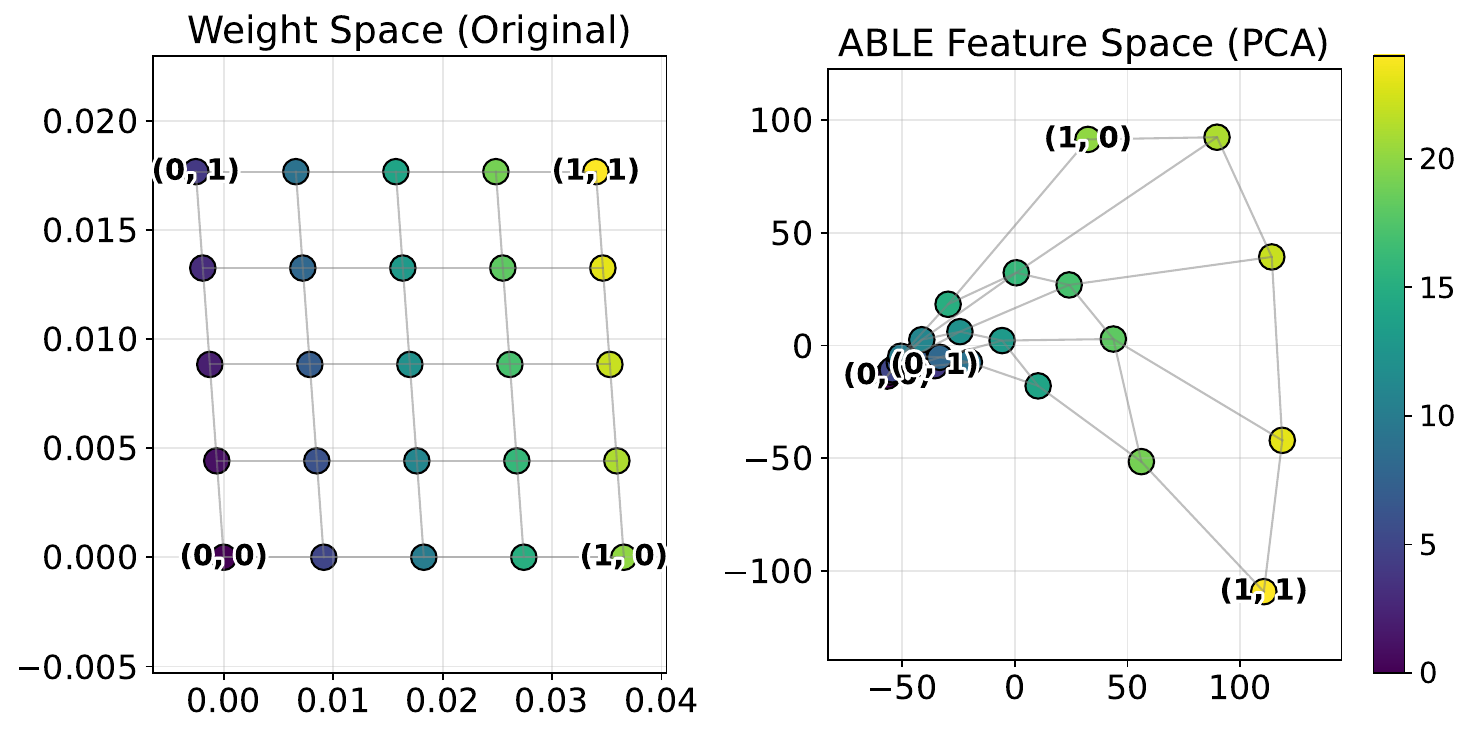}
\caption{Comparison of weight space and ABLE feature space. Left: weight space reconstructed via inner products. Right: ABLE feature space after PCA reduction. ABLE space preserves the topological structure of weight space but exhibits non-linear distortion. Notably, (0,0) and (0,1) are very close in ABLE space, indicating that Vicuna and the base model exhibit similar behavior, consistent with Vicuna's relatively small fine-tuning magnitude.}
\label{fig:interpolation}
\end{figure}

\begin{table*}[t]
\centering
\caption{Comparison of runtime and relation prediction accuracy. Relation prediction accuracy is reported as mean $\pm$ standard deviation over 5 random seeds.}
\label{tab:runtime}
\begin{tabular}{lcc}
\toprule
\textbf{Method} & \textbf{Total Time (h)} & \textbf{Accuracy on Relation Prediction} \\
\midrule
Log-Likelihood (Output) & 6.5 & 0.721$\pm$0.041 \\
PhyloLM (Output) & 9.8 & 0.704$\pm$0.141 \\
\textbf{ABLE (Ours)} & 18.5 & \textbf{0.867$\pm$0.042} \\
\bottomrule
\end{tabular}
\end{table*}

\begin{figure}[t]
\centering
\includegraphics[width=\columnwidth]{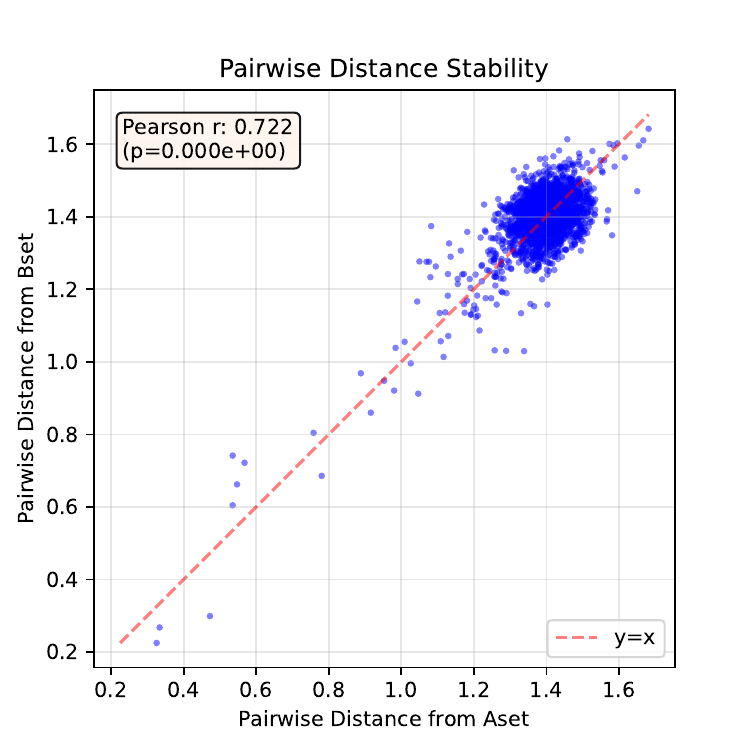}
\caption{Robustness test for ABLE embeddings. Each point corresponds to a model pair, with x-axis showing the pairwise distance computed from Subset A and y-axis from Subset B.}
\label{fig:robustness}
\end{figure}

As shown in Fig.~\ref{fig:interpolation}, a notable observation is that (0,0) and (0,1) are very close in ABLE space, while (0,0) and (1,0) are much farther apart. This indicates that Vicuna ($W_2$) exhibits higher behavioral similarity to the base model than Chat ($W_1$) does. This observation is consistent with the weight space structure: inner products computed directly from model weights show $\|W_2-W_0\|^2 \approx 0.00032$, whereas $\|W_1-W_0\|^2 \approx 0.00134$. Vicuna's fine-tuning magnitude is approximately one quarter that of Chat, demonstrating that ABLE distance reflects the similarity structure in parameter space.

Examining the overall geometry of both spaces, weight space exhibits a regular parallelogram grid reflecting the geometry of linear interpolation. Although ABLE feature space displays irregular distortion, it crucially preserves the topological structure consistent with weight space: (1) adjacent models remain neighbors in both spaces; (2) grid edges do not intersect, indicating that local neighborhood structure is preserved; (3) color gradients of model indices show similar transition patterns across both spaces.

This result validates the theoretical prediction at two levels. The topological preservation confirms the continuity of the ABLE mapping: models close in parameter space remain close in ABLE space. The geometric distortion reveals non-linear effects present in Transformers, indicating that ABLE sensitivity to parameter changes varies across different regions. This is consistent with the Lipschitz continuity established in Appendix~\ref{sec:theory}: ABLE guarantees bounded propagation of model similarity without requiring strict linear correspondence.

\section{Robustness to Sample Selection}
\label{sec:robustness}
To evaluate the robustness of ABLE to probing sample selection, we conduct a Mantel test. Specifically, we randomly partition the probing dataset into two non-overlapping subsets of equal size, and independently compute ABLE embeddings for the same set of models on each subset. We then compute the Euclidean distance between all pairs of models based on embeddings from each subset, yielding two sets of pairwise distances.

As shown in Fig.~\ref{fig:robustness}, the two sets of distances exhibit a strong positive correlation (Pearson $r = 0.722$, $p < 0.001$), with points clustering tightly around the $y = x$ diagonal. This result indicates that the inter-model distance structure captured by ABLE embeddings remains consistent across different probing samples, demonstrating that ABLE captures intrinsic behavioral characteristics of models rather than patterns dependent on specific probing inputs.

\paragraph{Leave-target-benchmark-out analysis.}
Four benchmark-score prediction targets---ARC, HellaSwag, MMLU, and WinoGrande---also contribute 200 questions each to the probing corpus. For each target, we reconstruct ABLE after removing its corresponding probing questions and repeat the same leave-one-model-out ridge-regression evaluation using the remaining 1,000 questions. 
The maximum absolute changes are 0.0230 for Pearson $r$ and 0.0237 for Spearman $\rho$ (Table \ref{tab:leave_target_out}). These results indicate that benchmark-score prediction is not primarily driven by direct probe--target overlap. Since all probing datasets use a multiple-choice format, this analysis does not establish robustness across substantially different task formats.

\begin{table}[t]
\centering
\small
\setlength{\tabcolsep}{3.5pt}
\caption{Benchmark-score prediction after removing probing questions
from the target benchmark.}
\label{tab:leave_target_out}
\begin{tabular}{lcccc}
\toprule
& \multicolumn{2}{c}{Pearson $r$}
& \multicolumn{2}{c}{Spearman $\rho$} \\
\cmidrule(lr){2-3}\cmidrule(lr){4-5}
\textbf{Target} & \textbf{Full} & \textbf{Removed}
& \textbf{Full} & \textbf{Removed} \\
\midrule
ARC        & 0.8781 & 0.8768 & 0.8898 & 0.8898 \\
HellaSwag  & 0.7764 & 0.7592 & 0.8157 & 0.8050 \\
MMLU       & 0.8554 & 0.8578 & 0.8637 & 0.8658 \\
WinoGrande & 0.8374 & 0.8144 & 0.8602 & 0.8365 \\
\bottomrule
\end{tabular}
\end{table}

\section{Computational Cost}
\label{sec:runtime}

We benchmark the computational time for 239 models using a server equipped with 8 $\times$ NVIDIA GeForce RTX 4090 GPUs. As shown in Table \ref{tab:runtime}, calculating ABLE features for the entire model set took 18.5 hours. While this involves a backward pass, the absolute time cost remains low and entirely feasible for large-scale auditing.

\begin{figure}[!ht]
\centering
\includegraphics[width=0.9\columnwidth]{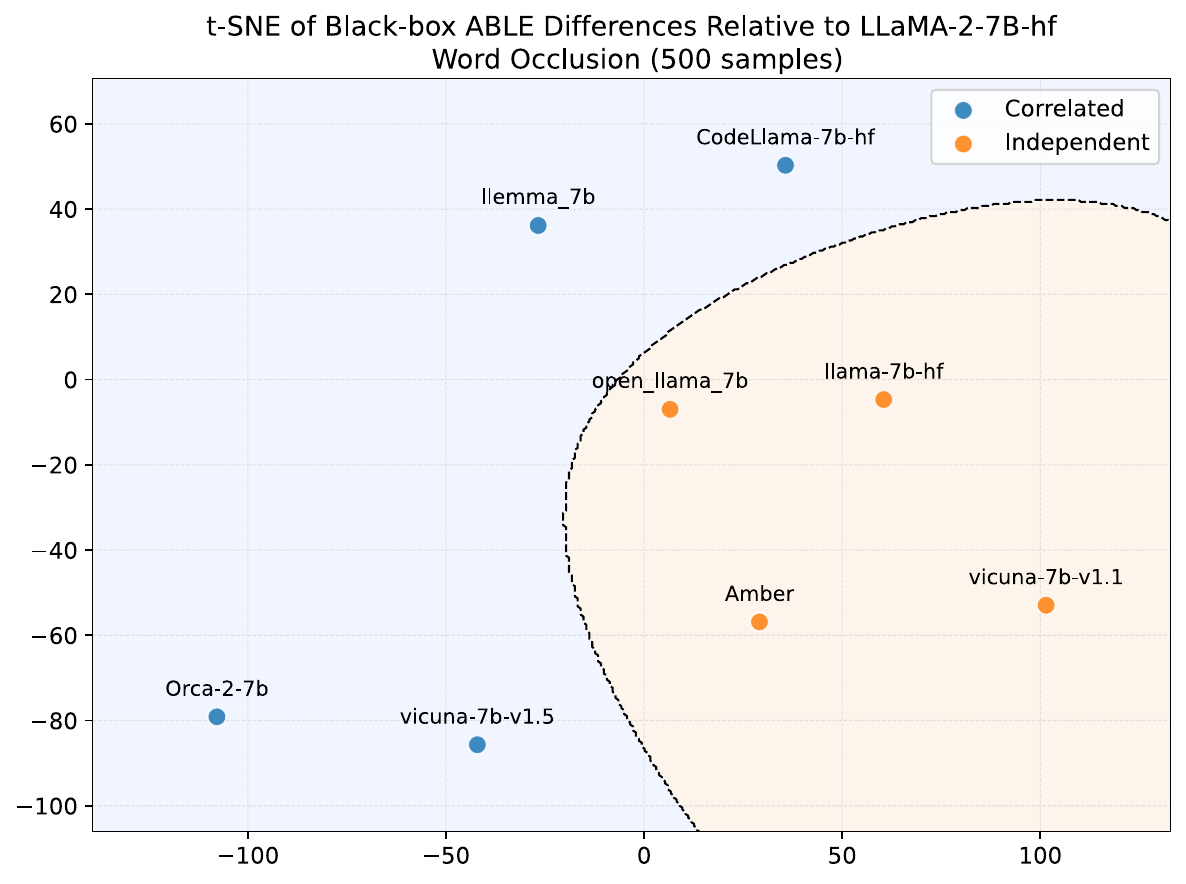}
\caption{Llama-2 family models in Black-box attribution method-based ABLE space, showing separation between annotated correlated and independent models
}
\label{fig:llama2_blackbox_500}
\end{figure}

Furthermore, we emphasize that this extraction is a one-time investment. Once the compact ABLE embedding is computed, it can be reused indefinitely for various downstream tasks without reloading the model weights. The increase in extraction time is justified by the significant performance gains (up to 23.1\% relative improvement in accuracy) achieved in relation prediction and other downstream applications.

\section{Perturbation-Based Attribution Validation on Llama-2 Family}
\label{sec:blackbox_attr}
To examine the attribution method flexibility of ABLE, we conduct a validation using perturbation-based black-box attributions on the Llama-2 family. 
Specifically, we use word-deletion occlusion, where token importance is estimated by measuring the change in answer-option log-probability after deleting words from the question. 
This setting does not require gradient access, allowing us to test whether ABLE can construct attribution-space representations beyond white-box attribution methods.
Using 500 probing samples, we construct black-box ABLE representations.

\begin{table*}[t]
  \centering
  \small
  \caption{
  Exploratory retrieval for models with incomplete or inconsistent
  provenance metadata. Parentheses report the median layer-wise relative
  parameter distance. ``n/a'' indicates that the retrieved pair has no
  compatible weight shapes and therefore cannot be assessed in parameter
  space.
  }
  \label{tab:undocumented_retrieval}
  \begin{tabular}{llll}
  \toprule
  \textbf{Query} & \textbf{ABLE top-1}
  & \textbf{Log-Likelihood top-1} & \textbf{PhyloLM top-1} \\
  \midrule
  Dolly-v2-3b
  & Pythia-2.8b (.010)
  & Pythia-2.8b (.010)
  & Pythia-2.8b-deduped (.152) \\

  Dolly-v2-7b
  & Pythia-6.9b (.010)
  & Pythia-2.8b-deduped (n/a)
  & Pythia-6.9b (.010) \\

  Dolly-v2-12b
  & Pythia-12b (.010)
  & Pythia-6.9b (n/a)
  & Pythia-410m (n/a) \\

  Experiment26-7B
  & Jaskier-7b-dpo-v5.6 (.001)
  & Jaskier-7b-dpo-v5.6 (.001)
  & Ogno-Jaskier-merge-7b (.021) \\

  UNA-SimpleSmaug-34b
  & Luminex-34B-v0.2 ($<10^{-7}$)
  & OpenHermes-2-Mistral-7B (n/a)
  & Luminex-34B-v0.2 ($<10^{-7}$) \\

  34b-beta
  & UNA-SimpleSmaug-34b (.174)
  & Luminex-34B-v0.2 (.174)
  & Luminex-34B-v0.1 (.175) \\

  Reflection-Llama-3.1-70B
  & Llama-3-70B-Instruct (.0055)
  & Llama-3-70B-Instruct (.0055)
  & Llama-3-70B-Instruct (.0055) \\
  \bottomrule
  \end{tabular}
  \end{table*}

\begin{table}[t]
\centering
\caption{Stability of black-box word occlusion-based ABLE representations across non-overlapping 100-sample subsets. \textbf{Upper triangle}: Pearson $r$; \textbf{Lower triangle}: Spearman $\rho$.}
\label{tab:occlusion-stability}
\small
\begin{tabular}{lccccc}
\toprule
& \textbf{S1} & \textbf{S2} & \textbf{S3} & \textbf{S4} & \textbf{S5} \\
\midrule
\textbf{S1} &  & 0.987 & 0.990 & 0.996 & 0.991 \\
\textbf{S2} & 0.929 &  & 0.993 & 0.973 & 0.992 \\
\textbf{S3} & 0.929 & 1.000 &  & 0.983 & 0.994 \\
\textbf{S4} & 1.000 & 0.929 & 0.929 &  & 0.985 \\
\textbf{S5} & 0.952 & 0.952 & 0.952 & 0.952 &  \\
\bottomrule
\end{tabular}
\end{table}

We repeat the Llama-2 relationship case study from Section~\ref{subsec:lineage}. 
As shown in Fig.~\ref{fig:llama2_blackbox_500}, the perturbation-based representations also separate the annotated correlated and independent models, suggesting that the attribution-space formulation is not tied to GI alone.

We further evaluate stability by splitting the 500 probing samples into five non-overlapping 100-sample subsets and independently constructing black-box ABLE representations from each subset. 
Table~\ref{tab:occlusion-stability} reports the blackbox-attribution based ABLEs are highly consistent, with mean Pearson $r=0.988$ and mean Spearman $\rho=0.952$, indicating that the word-occlusion signal is stable in this controlled setting.

\section{Candidate Relationship Retrieval beyond Documentation}
\label{sec:undocumented_retrieval}
Our main experiments evaluate whether ABLE recovers documented model relationships. We additionally consider six models from our collection whose provenance metadata were missing or inconsistent at the time of our audit. We also include Reflection-Llama-3.1-70B as an external query because its reported base model had been publicly questioned. For each query, we rank candidate models by cosine similarity between their ABLE embeddings and retrieve the top-ranked candidate, excluding the query itself. We also perform retrieval using Log-Likelihood and PhyloLM features. 

\begin{table}[t]
\centering
\caption{Task-specific candidate selection performance using ABLE embeddings.}
\label{tab:candidate_selection}
\small
\begin{tabular}{lccc}
\toprule
\textbf{Domain} & \textbf{Benchmark} & \textbf{NDCG@10} & \textbf{Pairwise Acc.} \\
\midrule
Math & GSM8K & 0.850 & 0.729 \\
Reasoning & ARC & 0.972 & 0.790 \\
Knowledge & MMLU & 0.855 & 0.762 \\
Truthfulness & TruthfulQA & 0.990 & 0.775 \\
\bottomrule
\end{tabular}
\end{table}

For each retrieved pair with compatible parameter shapes, we align two-dimensional weight tensors by name and shape and compute the median layer-wise relative distance
\begin{equation}
d_{\mathrm{rel}}(q,c) =
  \operatorname{median}_{\ell}
  \frac{\lVert W_q^{(\ell)}-W_c^{(\ell)}\rVert_2}
       {\lVert W_q^{(\ell)}\rVert_2}.
\end{equation}

ABLE retrieves a weight-compatible candidate in all seven cases (Table \ref{tab:undocumented_retrieval}). Six retrieved pairs exhibit parameter distances characteristic of close derivatives, while the 34b-beta case shows weaker but nontrivial parameter-space proximity. The output-based baselines recover the same
candidate in several cases, but also return candidates that cannot be compared directly in parameter space. These case studies therefore do not establish that ABLE universally outperforms the baselines. Instead, they show that ABLE can prioritize plausible candidates when metadata
alone is insufficient. 

\section{Additional Capability Ranking Analyses}
\label{sec:capability_ranking}

\begin{table*}[t]
\centering
\caption{Benchmark ranking prediction comparison using Kendall's Tau. Higher values indicate stronger agreement with reported benchmark-score rankings.}
\label{tab:ranking_baseline}
\small
\begin{tabular}{lccccccc}
\toprule
\textbf{Method} & \textbf{ARC} & \textbf{HellaSwag} & \textbf{MMLU} & \textbf{TruthfulQA} & \textbf{WinoGrande} & \textbf{GSM8K} & \textbf{Mean} \\
\midrule
PhyloLM & 0.514 & 0.601 & 0.557 & 0.260 & 0.439 & 0.536 & 0.485 \\
Log-Likelihood & 0.585 & 0.618 & 0.484 & 0.476 & 0.538 & 0.392 & 0.516 \\
\textbf{ABLE} & \textbf{0.664} & \textbf{0.649} & \textbf{0.655} & \textbf{0.605} & \textbf{0.626} & \textbf{0.602} & \textbf{0.634} \\
\bottomrule
\end{tabular}
\end{table*}

\begin{table*}[t]
\centering
\caption{Relation prediction performance on the L32\_H4096\_A32 architecture group (88 models, 50 positive pairs, 100 random negative pairs, 5 seeds).}
\label{tab:able-vs-mother-l32}
\small
\begin{tabular}{lccccc}
\toprule
\textbf{Method} & \textbf{Accuracy} & \textbf{Precision} & \textbf{Recall} & \textbf{F1} & \textbf{AUC} \\
\midrule
\textbf{ABLE} & \textbf{0.913{\scriptsize $\pm$0.016}} & \textbf{0.939{\scriptsize $\pm$0.067}} &
0.840{\scriptsize $\pm$0.063} & \textbf{0.866{\scriptsize $\pm$0.039}} & 0.942{\scriptsize $\pm$0.045} \\
\textbf{MoTHer} & 0.880{\scriptsize $\pm$0.058} & 0.763{\scriptsize $\pm$0.098} &
\textbf{0.960{\scriptsize $\pm$0.049}} & 0.846{\scriptsize $\pm$0.066} & \textbf{0.974{\scriptsize
$\pm$0.024}} \\
\bottomrule
\end{tabular}
\end{table*}

We evaluate whether ABLE embeddings can support capability-oriented model ranking. 
First, we consider task-specific candidate selection to construct a shortlist of promising models for a target capability before expensive fine-tuning. 
We evaluate four domains using GSM8K, ARC, MMLU, and TruthfulQA as target benchmarks. 
For each domain, we rank models using predictions derived from ABLE embeddings and compare the ranking with reported benchmark scores using NDCG@10 and pairwise ranking accuracy.
As shown in Table~\ref{tab:candidate_selection}, ABLE achieves strong top-10 ranking quality, and pairwise ranking accuracy ranging from 0.729 to 0.790. 
These results suggest that ABLE embeddings can serve as a lightweight signal for preliminary task-specific candidate selection.

Second, we compare ABLE with output-based representation baselines for benchmark ranking prediction. 
We use Kendall's Tau to measure agreement between predicted model rankings and reported benchmark-score rankings across six benchmarks.
As shown in Table~\ref{tab:ranking_baseline}, ABLE obtains the highest Kendall's Tau across all six benchmarks.

\section{Adaptation of Unified Topological Signatures}
\label{sec:uts_adaptation}
Unified Topological Signatures (UTS)~\citep{rottach2025topology} was originally designed for text‑embedding models. To enable comparison with our proposed method, we adapt it to decoder‑only autoregressive LLMs. Each model encodes the same 5,183 SciFact documents \cite{wadden2020fact}; we mean‑pool the non‑padding final‑layer hidden states and L2‑normalize the resulting document vectors. We then apply the released implementation to obtain a 16‑dimensional UTS signature for each model. UTS is evaluated using the same pair construction, RBF‑SVM classifier, and five data splits as ABLE. This is a single‑corpus adaptation rather than a direct reproduction of the original UTS evaluation.

\section{Comparison with the Parameter-Based Baseline}
\label{sec:mother_comparison}
To compare with parameter-based methods, we conduct an additional head-to-head evaluation against MoTHer \cite{horwitz2025unsupervised}. 
Since MoTHer requires layer-wise weight matrices to have compatible dimensions, this comparison must be restricted to models with the same architecture. 
We therefore select the L32\_H4096\_A32 architecture group, which contains 88 models with 32 layers, hidden size 4096, and 32 attention heads, including Mistral-7B, Llama-2-7B, Llama-3-8B, CodeLlama-7B, and their fine-tuned variants. 
This setting is favorable to MoTHer because all compared models have compatible weight shapes.

We construct a relation prediction dataset within this architecture group. 
Positive pairs are collected from known model relationships in Hugging Face model cards, including fine-tuning, adapter, merging, and RL-related relationships. 
Negative pairs are sampled from model pairs without known relationships in the same architecture group. 
In total, the dataset contains 50 positive pairs and 100 negative pairs. 
We evaluate both methods using the same stratified 80/20 train/test splits over 5 random seeds. Both methods are evaluated with the same RBF-kernel SVM classifier.

As shown in Table~\ref{tab:able-vs-mother-l32}, ABLE achieves higher accuracy, precision, and F1, while MoTHer obtains higher recall and AUC. 
This shows that ABLE remains competitive with a parameter-based baseline in a same-architecture setting favorable to weight comparison. 
More broadly, ABLE and MoTHer serve different purposes. MoTHer is a specialized parameter-based method designed for recovering model tree relationships under compatible architectures, whereas ABLE provides a general model representation that can be reused across heterogeneous architectures and tokenizers.

\section{Model Selection and Model List}
\label{sec:model_list}
The model collection is a manually curated, purposive sample rather than a probability sample of the entire model ecosystem. A model was eligible if it (1) was an open-weight, decoder-only autoregressive language model, and (2) had an accessible checkpoint and tokenizer when the ABLE representations were constructed. Among eligible models, we sought broad coverage of model families, parameter scales, release years, and training types, including base, instruction-tuned, merged, and domain-adapted models. We additionally included cohorts with documented base--derivative relationships because they provide auditable positive examples for relation prediction. No model was included or excluded based on its ABLE embedding or downstream performance.

The resulting collection contains 239 models from 24 families, spanning 0.07B--72B parameters and releases from 2022 to 2025 (Table~\ref{tab:models}). It is not intended to be statistically representative of all publicly released language models. The accessibility and gradient requirements exclude
closed-source models and models whose checkpoints or implementations could not be evaluated through the common pipeline. The collection also overrepresents widely available families, common parameter scales, and models with documented derivation histories. Consequently, aggregate patterns in the model atlas should be interpreted as properties of this curated collection rather than estimates of model prevalence in the broader ecosystem.

\clearpage
\onecolumn
{\footnotesize
\topcaption{Complete list of models included in this study.}\label{tab:models}
\tablefirsthead{%
\toprule
\textbf{Model Name} & \textbf{Size} & \textbf{Model Type} & \textbf{Downloads} \\
\midrule}
\tablehead{%
\multicolumn{4}{l}{\small\sl continued from previous page}\\
\toprule
\textbf{Model Name} & \textbf{Size} & \textbf{Model Type} & \textbf{Downloads} \\
\midrule}
\tabletail{%
\midrule
\multicolumn{4}{r}{\small\sl continued on next page}\\}
\tablelasttail{}
\begin{center}
\begin{supertabular}{p{7cm}p{1.5cm}p{2cm}r}
\input{model_table_rows_single.tex}
\end{supertabular}
\end{center}
}
\clearpage
\twocolumn

\end{document}

%% file: model_table_rows_single.tex
01-ai/Yi-34B-Chat & 34B & Llama & 30,397 \\
01-ai/Yi-6B & 6B & Llama & 16,054 \\
01-ai/Yi-6B-200K & 6B & Llama & 18,155 \\
abocide/Qwen2.5-7B-Instruct-R1-forfinance & 7B & Qwen2 & 529 \\
AdaptLLM/medicine-chat & 6.57B & Llama & 1,244 \\
AdaptLLM/medicine-LLM & 6.57B & Llama & 168 \\
AdaptLLM/medicine-LLM-13B & 13B & Llama & 34 \\
allenai/tulu-2-dpo-70b & 70B & Llama & 2,543 \\
Arc53/docsgpt-7b-mistral & 7B & Mistral & 89 \\
augmxnt/shisa-base-7b-v1 & 7B & Mistral & 1,041 \\
bardsai/jaskier-7b-dpo-v5.6 & 7B & Mistral & 159 \\
berkeley-nest/Starling-LM-7B-alpha & 7B & Mistral & 1,493 \\
bigcode/octocoder & 15.5B & Other & 148 \\
bigscience/bloom-1b1 & 1.1B & Bloom & 6,749 \\
bigscience/bloom-1b7 & 1.7B & Bloom & 28,490 \\
bigscience/bloom-3b & 3B & Bloom & 8,904 \\
bigscience/bloom-560m & 0.56B & Bloom & 112,394 \\
bigscience/bloom-7b1 & 7B & Bloom & 11,136 \\
bigscience/bloomz-3b & 3B & Bloom & 4,235 \\
bigscience/bloomz-560m & 0.56B & Bloom & 846,321 \\
bigscience/bloomz-7b1 & 7B & Bloom & 5,299 \\
Biomimicry-AI/ANIMA-Nectar-v2 & 6.57B & Mistral & 975 \\
BioMistral/BioMistral-7B & 7B & Mistral & 120,331 \\
BioMistral/BioMistral-7B-DARE & 7B & Mistral & 1,904 \\
bxod/Llama-3.2-1B-Instruct-uz & 1B & Llama & 35 \\
CausalLM/34b-beta & 34B & Llama & 8,349 \\
cerebras/Cerebras-GPT-1.3B & 1.3B & GPT-2 & 1,266 \\
cerebras/Cerebras-GPT-111M & 0.11B & GPT-2 & 4,694 \\
cerebras/Cerebras-GPT-2.7B & 2.7B & GPT-2 & 977 \\
cerebras/Cerebras-GPT-256M & 0.26B & GPT-2 & 1,185 \\
cerebras/Cerebras-GPT-590M & 0.59B & GPT-2 & 1,088 \\
cerebras/Cerebras-GPT-6.7B & 6.7B & GPT-2 & 943 \\
cloudyu/Mixtral\_11Bx2\_MoE\_19B & 19B & Mixtral & 920 \\
codefuse-ai/CodeFuse-DeepSeek-33B & 33B & Llama & 127 \\
codellama/CodeLlama-13b-Instruct-hf & 13B & Llama & 19,108 \\
codellama/CodeLlama-34b-Instruct-hf & 34B & Llama & 20,472 \\
codellama/CodeLlama-7b-hf & 7B & Llama & 53,172 \\
cognitivecomputations/yayi2-30b-llama & 30B & Llama & 52 \\
continuedev/instinct & 4.86B & Qwen2 & 241 \\
ConvexAI/Luminex-34B-v0.1 & 34B & Llama & 7,565 \\
ConvexAI/Luminex-34B-v0.2 & 34B & Llama & 7,640 \\
CorticalStack/pastiche-crown-clown-7b-dare-dpo & 7B & Mistral & 71 \\
CultriX/NeuralTrix-bf16 & 6.57B & Mistral & 74 \\
databricks/dolly-v2-12b & 11.58B & GPT-NeoX & 3,047 \\
databricks/dolly-v2-3b & 3B & GPT-NeoX & - \\
databricks/dolly-v2-7b & 7B & GPT-NeoX & - \\
davanstrien/query-gen & 8B & Llama & 56 \\
DeepHat/DeepHat-V1-7B & 7B & Qwen2 & 1,599 \\
deepseek-ai/deepseek-coder-1.3b-base & 1.3B & Llama & 21,029 \\
deepseek-ai/deepseek-coder-6.7b-instruct & 6.7B & Llama & 44,962 \\
deepseek-ai/deepseek-llm-67b-chat & 67B & Llama & 1,815 \\
deepseek-ai/deepseek-math-7b-instruct & 7B & Llama & 5,002 \\
dfurman/HermesBagel-34B-v0.1 & 34B & Llama & 96 \\
EleutherAI/gpt-neo-1.3B & 1.3B & GPT-Neo & 28,899 \\
EleutherAI/gpt-neo-125m & 0.12B & GPT-Neo & 111,607 \\
EleutherAI/gpt-neo-2.7B & 2.7B & GPT-Neo & 16,759 \\
EleutherAI/llemma\_34b & 34B & Llama & 232 \\
EleutherAI/llemma\_7b & 7B & Llama & 820 \\
EleutherAI/pythia-1.4b & 1.4B & GPT-NeoX & 18,921 \\
EleutherAI/pythia-1.4b-deduped & 1.4B & GPT-NeoX & 8,043 \\
EleutherAI/pythia-12b & 12B & GPT-NeoX & 15,377 \\
EleutherAI/pythia-160m & 0.16B & GPT-NeoX & 89,343 \\
EleutherAI/pythia-1b-deduped & 1B & GPT-NeoX & 9,688 \\
EleutherAI/pythia-2.8b & 2.8B & GPT-NeoX & 23,704 \\
EleutherAI/pythia-2.8b-deduped & 2.8B & GPT-NeoX & 7,199 \\
EleutherAI/pythia-410m & 0.41B & GPT-NeoX & 43,070 \\
EleutherAI/pythia-6.9b & 6.9B & GPT-NeoX & 17,578 \\
EleutherAI/pythia-70m & 0.07B & GPT-NeoX & 137,323 \\
eren23/ogno-monarch-jaskier-merge-7b-OH-PREF-DPO & 7B & Mistral & 71 \\
facebook/opt-1.3b & 1.3B & OPT & 351,633 \\
facebook/opt-125m & 0.12B & OPT & 4,242,012 \\
facebook/opt-13b & 13B & OPT & 9,457 \\
facebook/opt-2.7b & 2.7B & OPT & 14,711 \\
facebook/opt-350m & 0.35B & OPT & 100,491 \\
facebook/opt-6.7b & 6.7B & OPT & 16,057 \\
fblgit/UNA-SimpleSmaug-34b-v1beta & 34B & Llama & 7,579 \\
FelixChao/llama2-13b-math1.2 & 13B & Llama & 1,093 \\
FelixChao/Scorpio-7B & 7B & Mistral & 58 \\
FelixChao/vicuna-7B-chemical & 7B & Llama & 1,086 \\
FelixChao/vicuna-7B-physics & 7B & Llama & 1,089 \\
golaxy/gowizardlm & 6.72B & Llama & 950 \\
google/codegemma-1.1-7b-it & 7B & Gemma & 150 \\
google/codegemma-2b & 2B & Gemma & 2,143 \\
google/codegemma-7b & 7B & Gemma & 2,006 \\
google/gemma-2b & 2B & Gemma & 180,941 \\
google/gemma-2b-it & 2B & Gemma & 59,882 \\
google/gemma-7b & 7B & Gemma & 54,233 \\
google/gemma-7b-it & 7B & Gemma & 108,191 \\
Harshvir/Llama-2-7B-physics & 7B & Llama & 1,024 \\
HuggingFaceH4/zephyr-7b-alpha & 7B & Mistral & 1,817 \\
HuggingFaceH4/zephyr-7b-beta & 7B & Mistral & 75,649 \\
ibivibiv/alpaca-dragon-72b-v1 & 72B & Llama & 877 \\
Imran1/MedChat3.5 & 7B & Mistral & 37 \\
inflatebot/MN-12B-Mag-Mell-R1 & 12B & Mistral & 429 \\
Intel/neural-chat-7b-v3-3 & 7B & Mistral & 31,866 \\
janhq/Jan-v1-4B & 4B & Qwen3 & 1,350 \\
janhq/Jan-v1-edge & 1.72B & Qwen3 & 51 \\
jiawei-ucas/Qwen-2.5-7B-ConsistentChat & 7B & Qwen2 & 16 \\
kamrr/llama-3-8b\_dolly\_lora & 8B & Llama & 44 \\
kevin009/llamaRAGdrama & 6.57B & Mistral & 809 \\
kloodia/alpaca & 8B & Llama & 82 \\
kloodia/lora-8b-alpaca-french & 8B & Llama & 87 \\
kloodia/lora-8b-bio & 8B & Llama & 101 \\
kloodia/lora-8b-code & 8B & Llama & 133 \\
kloodia/lora-8b-math & 8B & Llama & 93 \\
kloodia/lora-8b-medic & 8B & Llama & 109 \\
kloodia/lora-8b-physic & 8B & Llama & 123 \\
kyujinpy/Sakura-SOLRCA-Math-Instruct-DPO-v1 & 9.79B & Llama & 1,008 \\
LLM360/Amber & 7B & Llama & 3,443 \\
lmsys/vicuna-13b-v1.5 & 13B & Llama & 116,184 \\
lmsys/vicuna-13b-v1.5-16k & 13B & Llama & 10,289 \\
lmsys/vicuna-33b-v1.3 & 33B & Llama & 1,403 \\
lmsys/vicuna-7b-v1.1 & 7B & Llama & 2,310 \\
lmsys/vicuna-7b-v1.5 & 7B & Llama & 162,278 \\
lmsys/vicuna-7b-v1.5-16k & 7B & Llama & 4,445 \\
MaziyarPanahi/WizardLM-Math-70B-v0.1 & 70B & Llama & 76 \\
meta-llama/Llama-2-13b-chat-hf & 13B & Llama & 209,634 \\
meta-llama/Llama-2-70b-chat-hf & 70B & Llama & 8,026 \\
meta-llama/Llama-2-7b-chat-hf & 7B & Llama & 325,422 \\
meta-llama/Llama-2-7b-hf & 7B & Llama & 599,490 \\
meta-llama/Llama-3.1-8B-Instruct & 8B & Llama & 10,639,638 \\
meta-llama/LlamaGuard-7b & 7B & Llama & 1,720 \\
meta-llama/Meta-Llama-3-70B & 70B & Llama & 480,747 \\
meta-llama/Meta-Llama-3-70B-Instruct & 70B & Llama & 56,883 \\
meta-llama/Meta-Llama-3-8B & 8B & Llama & 2,236,547 \\
meta-llama/Meta-Llama-3-8B-Instruct & 8B & Llama & 1,531,726 \\
meta-llama/Meta-Llama-Guard-2-8B & 8B & Llama & 15,469 \\
meta-math/MetaMath-Llemma-7B & 7B & Llama & 1,033 \\
meta-math/MetaMath-Mistral-7B & 7B & Mistral & 2,201 \\
microsoft/MediPhi-Instruct & 3.72B & Phi-3 & 2,007 \\
microsoft/Orca-2-13b & 13B & Llama & 10,627 \\
microsoft/Orca-2-7b & 7B & Llama & 9,125 \\
microsoft/phi-1\_5 & 1.5B & Phi & 46,747 \\
microsoft/phi-2 & 2.65B & Phi & 1,071,489 \\
microsoft/Phi-3.5-mini-instruct & 3.72B & Phi-3 & 335,634 \\
mistralai/Mistral-7B-Instruct-v0.1 & 7B & Mistral & 491,519 \\
mistralai/Mistral-7B-v0.1 & 7B & Mistral & 355,407 \\
mlabonne/AlphaMonarch-7B & 7B & Mistral & 12,378 \\
mlabonne/NeuralHermes-2.5-Mistral-7B & 7B & Mistral & 105 \\
mosaicml/mpt-30b-instruct & 30B & MPT & 2,646 \\
mosaicml/mpt-7b & 7B & MPT & - \\
mosaicml/mpt-7b-8k & 7B & MPT & - \\
mosaicml/mpt-7b-8k-chat & 7B & MPT & - \\
mosaicml/mpt-7b-chat & 7B & MPT & 81,083 \\
mosaicml/mpt-7b-instruct & 7B & MPT & - \\
Neko-Institute-of-Science/metharme-7b & 7B & Llama & 1,061 \\
Neko-Institute-of-Science/pygmalion-7b & 7B & Llama & 1,102 \\
neovalle/H4rmoniousAnthea & 7B & Mistral & 87 \\
Nexusflow/Starling-LM-7B-beta & 7B & Mistral & 1,434 \\
nomic-ai/gpt4all-13b-snoozy & 13B & Llama & 1,012 \\
NousResearch/Hermes-4-14B & 14B & Qwen3 & 5,100 \\
NousResearch/Nous-Hermes-13b & 13B & Llama & 1,330 \\
NousResearch/Nous-Hermes-2-Yi-34B & 34B & Llama & 7,907 \\
NousResearch/Nous-Hermes-llama-2-7b & 7B & Llama & 1,253 \\
OpenAssistant/oasst-sft-4-pythia-12b-epoch-3.5 & 12B & GPT-NeoX & 1,519 \\
OpenBuddy/openbuddy-codellama2-34b-v11.1-bf16 & 34B & Llama & 1,143 \\
openchat/openchat-3.5-0106 & 6.57B & Mistral & 11,448 \\
openchat/openchat\_3.5 & 6.57B & Mistral & 2,622 \\
openlm-research/open\_llama\_7b & 7B & Llama & 18,239 \\
pbevan11/llama-3-8b-ocr-correction & 8B & Llama & 64 \\
PharMolix/BioMedGPT-LM-7B & 7B & Llama & 578 \\
Plaban81/Moe-4x7b-math-reason-code & 7B & Mixtral & 67 \\
prithivMLmods/rStar-Coder-Qwen3-0.6B & 0.6B & Qwen3 & 21 \\
project-baize/baize-v2-13b & 13B & Llama & 1,379 \\
Q-bert/Optimus-7B & 7B & Mistral & 1,106 \\
Qwen/Qwen1.5-0.5B & 0.5B & Qwen2 & 46,541 \\
Qwen/Qwen1.5-0.5B-Chat & 0.5B & Qwen2 & 57,690 \\
Qwen/Qwen1.5-1.8B & 1.8B & Qwen2 & 17,565 \\
Qwen/Qwen1.5-1.8B-Chat & 1.8B & Qwen2 & 89,028 \\
Qwen/Qwen1.5-14B & 14B & Qwen2 & 50,314 \\
Qwen/Qwen1.5-14B-Chat & 14B & Qwen2 & 15,409 \\
Qwen/Qwen1.5-32B & 32B & Qwen2 & 11,415 \\
Qwen/Qwen1.5-32B-Chat & 32B & Qwen2 & 38,786 \\
Qwen/Qwen1.5-4B & 4B & Qwen2 & 13,977 \\
Qwen/Qwen1.5-4B-Chat & 4B & Qwen2 & 17,665 \\
Qwen/Qwen1.5-7B & 7B & Qwen2 & 70,483 \\
Qwen/Qwen1.5-7B-Chat & 7B & Qwen2 & 16,435 \\
Qwen/Qwen2-0.5B & 0.5B & Qwen2 & 329,782 \\
Qwen/Qwen2-1.5B & 1.5B & Qwen2 & 212,957 \\
Qwen/Qwen2.5-0.5B-Instruct & 0.5B & Qwen2 & 2,456,001 \\
Qwen/Qwen2.5-1.5B & 1.5B & Qwen2 & 504,513 \\
Qwen/Qwen2.5-1.5B-Instruct & 1.5B & Qwen2 & 4,990,149 \\
Qwen/Qwen2.5-7B & 7B & Qwen2 & 821,385 \\
Qwen/Qwen2.5-7B-Instruct & 7B & Qwen2 & 6,379,230 \\
Qwen/Qwen2.5-Coder-7B & 7B & Qwen2 & 91,843 \\
Qwen/Qwen2.5-Coder-7B-Instruct & 7B & Qwen2 & 553,942 \\
Qwen/Qwen3-0.6B & 0.6B & Qwen3 & 8,292,645 \\
Qwen/Qwen3-0.6B-Base & 0.6B & Qwen3 & 161,046 \\
Qwen/Qwen3-1.7B & 1.7B & Qwen3 & 5,473,707 \\
Qwen/Qwen3-14B & 14B & Qwen3 & 686,547 \\
Qwen/Qwen3-4B & 4B & Qwen3 & 3,966,909 \\
Qwen/Qwen3-4B-Instruct-2507 & 4B & Qwen3 & 4,166,985 \\
Qwen/Qwen3-4B-Thinking-2507 & 4B & Qwen3 & 501,267 \\
Qwen/Qwen3-8B & 8B & Qwen3 & 4,376,885 \\
rishiraj/CatPPT-base & 6.57B & Mistral & 3,740 \\
rubenamtz0/llama-3-8b-lora-law2entity & 8B & Llama & 47 \\
sail/Sailor-7B & 7B & Qwen2 & 100 \\
scb10x/typhoon-7b & 7B & Mistral & 36,854 \\
SciPhi/SciPhi-Mistral-7B-32k & 7B & Mistral & 1,098 \\
SciPhi/SciPhi-Self-RAG-Mistral-7B-32k & 7B & Mistral & 1,241 \\
shleeeee/mistral-ko-tech-science-v1 & 6.57B & Mistral & 13 \\
stabilityai/stable-code-3b & 3B & StableLM & 3,617 \\
stabilityai/stablelm-2-1\_6b & 6B & StableLM & 1,829 \\
stabilityai/stablelm-3b-4e1t & 3B & StableLM & 18,150 \\
stabilityai/stablelm-base-alpha-3b & 3B & GPT-NeoX & 2,084 \\
stabilityai/stablelm-base-alpha-7b & 7B & GPT-NeoX & 1,980 \\
stabilityai/stablelm-tuned-alpha-3b & 3B & GPT-NeoX & 2,124 \\
stabilityai/stablelm-tuned-alpha-7b & 7B & GPT-NeoX & 2,393 \\
stabilityai/stablelm-zephyr-3b & 3B & StableLM & 9,110 \\
SUSTech/SUS-Chat-34B & 34B & Llama & 998 \\
teknium/OpenHermes-2-Mistral-7B & 7B & Mistral & 402 \\
teknium/OpenHermes-2.5-Mistral-7B & 7B & Mistral & 202,048 \\
tenyx/TenyxChat-7B-v1 & 7B & Mistral & 889 \\
Tesslate/UIGEN-X-4B-0729 & 4B & Qwen3 & 30 \\
Tesslate/WEBGEN-4B-Preview & 4B & Qwen3 & 55 \\
TheBloke/koala-13B-HF & 13B & Llama & 1,469 \\
TheBloke/tulu-30B-fp16 & 30B & Llama & 1,174 \\
theprint/TiTan-Qwen2.5-0.5B & 0.5B & Qwen2 & - \\
TigerResearch/tigerbot-13b-base & 13B & Llama & 54 \\
TigerResearch/tigerbot-7b-base-v1 & 7B & Bloom & 18 \\
TigerResearch/tigerbot-7b-base-v2 & 7B & Bloom & 39 \\
TigerResearch/tigerbot-7b-sft-v1 & 7B & Bloom & 47 \\
TigerResearch/tigerbot-7b-sft-v2 & 7B & Bloom & 32 \\
tiiuae/falcon-40b-instruct & 40B & Falcon & 50,635 \\
tiiuae/falcon-7b & 7B & Falcon & 64,169 \\
tiiuae/falcon-7b-instruct & 7B & Falcon & 52,800 \\
tiiuae/falcon-rw-1b & 1B & Falcon & 6,477 \\
tomg-group-umd/DynaGuard-8B & 8B & Qwen3 & 124 \\
upstage/SOLAR-10.7B-Instruct-v1.0 & 10.7B & Llama & 28,257 \\
venetis/llama3-8b-hermes-sandals-100 & 8B & Llama & 52 \\
VinitT/Sanskrit-llama & 8B & Llama & 64 \\
Vortex5/Lunar-Nexus-12B & 12B & Mistral & 19 \\
Vortex5/Moonlit-Shadow-12B & 12B & Mistral & 14 \\
WizardLM/WizardLM-70B-V1.0 & 70B & Llama & 17,901 \\
worldboss/llama-3-8b-axolotl-fine-tune-qlora & 8B & Llama & 70 \\
Writer/palmyra-med-20b & 20B & GPT-2 & 1,153 \\
yahma/llama-7b-hf & 7B & Llama & 4,983 \\
yam-peleg/Experiment26-7B & 7B & Mistral & 173 \\
zhengr/MixTAO-7Bx2-MoE-v8.1 & 6.57B & Mixtral & 19,428 \\
\specialrule{\heavyrulewidth}{0pt}{0pt}